\documentclass{article}
\usepackage[preprint]{colm2026_conference}

\usepackage{microtype}
\usepackage{hyperref}
\usepackage{url}
\usepackage{booktabs}
\usepackage{amsmath,amssymb,amsthm}
\usepackage{graphicx}
\usepackage{subcaption}
\usepackage{xcolor}
\usepackage{algorithm}
\usepackage{algpseudocode}
\usepackage{multirow}
\usepackage{makecell}
\usepackage{lineno}
\usepackage{wrapfig}
\usepackage{enumitem}
\usepackage{colortbl}
\usepackage{float}
\usepackage{tikz}
\usetikzlibrary{arrows.meta,positioning,shapes.geometric,backgrounds,fit,calc}

\definecolor{clrStart}{RGB}{168,214,168}
\definecolor{clrFwd}  {RGB}{209,228,250}
\definecolor{clrDetect}{RGB}{248,210,210}
\definecolor{clrAction}{RGB}{255,238,186}
\definecolor{clrEmit}  {RGB}{225,225,225}

\tikzset{
  startstop/.style={ellipse,draw=black!70,thick,fill=clrStart,
    font=\small\bfseries,align=center,
    minimum width=2.4cm,minimum height=1.1cm},
  proc/.style={rectangle,rounded corners=5pt,draw=black!70,thick,
    align=center,text width=3.5cm,minimum height=1.3cm,
    font=\small,inner sep=5pt},
  fwd/.style={proc,fill=clrFwd},
  det/.style={proc,fill=clrDetect},
  act/.style={proc,fill=clrAction},
  emt/.style={proc,fill=clrEmit},
  badge/.style={circle,draw=black!70,fill=white,thick,
    font=\footnotesize\bfseries,inner sep=1pt,minimum size=0.46cm},
  arr/.style={-{Stealth[length=6pt,width=4.5pt]},thick,color=black!60},
  lbl/.style={font=\footnotesize\itshape,color=black!55,
              fill=white,inner sep=1.5pt},
}

\definecolor{darkblue}{rgb}{0,0,0.5}
\hypersetup{colorlinks=true,citecolor=darkblue,linkcolor=darkblue,urlcolor=darkblue}

\newcommand{\lcrit}{\ell_{\mathrm{crit}}}
\newcommand{\tauflip}{\tau_\phi}
\newcommand{\tauH}{\tau_H}
\newcommand{\alphamax}{\alpha_{\max}}
\newcommand{\Vbasis}{\mathcal{V}}
\newcommand{\dopt}{\delta^*}

\newtheorem{theorem}{Theorem}[section]
\newtheorem{proposition}[theorem]{Proposition}
\newtheorem{lemma}[theorem]{Lemma}

\theoremstyle{definition}
\newtheorem{definition}[theorem]{Definition}
\theoremstyle{remark}

\ifcolmsubmission
\linenumbers
\fi

\title{Latent Phase-Shift Rollback: Inference-Time Error Correction via\\
Residual Stream Monitoring and KV-Cache Steering}

\author{
  Manan Gupta\textsuperscript{1} \quad
  Dhruv Kumar\textsuperscript{1} \\
  \\
  \textsuperscript{1}BITS Pilani, Pilani Campus, India \\
  \texttt{\{f20241231, dhruv.kumar\}@pilani.bits-pilani.ac.in}
}

\begin{document}
\maketitle

\begin{abstract}
Large language models frequently commit unrecoverable reasoning errors
mid-generation: once a wrong step is taken, subsequent tokens compound
the mistake rather than correct it.
We introduce \textbf{Latent Phase-Shift Rollback} (LPSR): at each
generation step, we monitor the residual stream at a critical layer
$\lcrit$, detect abrupt directional reversals (\emph{phase shifts}) via
a cosine-similarity $+$ entropy dual gate, and respond by rolling back
the KV-cache and injecting a pre-computed steering vector.
No fine-tuning, gradient computation, or additional forward passes are
required.
LPSR achieves $\mathbf{44.0\%}$ on MATH-500 with an 8B model versus
$28.8\%$ for standard AR ($+15.2$ pp; McNemar $\chi^2 = 66.96$,
$p < 10^{-15}$).
Critically, prompted self-correction, the most natural inference-time
baseline, scores only $19.8\%$, \emph{below} standard AR; LPSR exceeds
it by $+24.2$ pp ($\chi^2 = 89.4$, $p \approx 0$).
LPSR also outperforms Best-of-16 ($+7.8$ pp) at $5.4\times$ lower token
cost, and surpasses a standard 70B model ($35.2\%$) with $8.75\times$
fewer parameters at ${\sim}3\times$ the token budget.
A 32-layer sweep reveals a novel \textbf{detection--correction
dissociation}: error-detection AUC peaks at layer~14 ($0.718$) but
task accuracy peaks at layer~16 ($44.0\%$ vs.\ $29.2\%$),
demonstrating that optimal monitoring depth differs for detection and
correction.
\end{abstract}

\section{Introduction}
\label{sec:intro}

The ability to perform multi-step reasoning remains a fundamental challenge
for large language models.
Under greedy decoding, Llama-3-8B fails on $71.2\%$ of MATH-500 problems;
a standard Llama-3-70B model fails on $64.8\%$, an improvement of only
$6.4$~pp despite $8.75\times$ more parameters.
A key failure mode is \emph{error propagation}: elementary mistakes such as sign
errors, wrong formula applications, variable confusions, compound undetected
over hundreds of subsequent tokens \citep{Wei2022,Lightman2024}.
Existing remedies fall into two categories: \emph{training-time} approaches
(fine-tuning on process-supervised reward models
\citep{Lightman2024,Uesato2022}) and \emph{inference-time} approaches
(chain-of-thought prompting \citep{Wei2022}, self-consistency
\citep{Wang2023}, Best-of-$N$ sampling \citep{Snell2024}, or tree search
\citep{Yao2023}).
Training-time methods are expensive and tied to a fixed model.
Inference-time methods either ignore the model's internal state entirely
(Best-of-$N$) or require many additional forward passes (tree search,
self-consistency).

We ask: \emph{Can we detect a reasoning error while it is forming, before
it has fully propagated, and correct the trajectory with minimal additional
computation?}

The internal geometry of transformer residual streams offers an answer.
\citet{Elhage2021} and \citet{Anthropic2022} showed that the residual
stream at middle layers encodes semantic content that is manipulable via
linear steering vectors.
Recent mechanistic interpretability work \citep{Zou2024,Turner2023}
demonstrates that these directions can be found unsupervised and generalise
across tasks.
Our observation is simpler: when a model is about to commit a reasoning
error, the direction of the residual stream at layer $\lcrit \approx L/2$
undergoes a sharp \emph{phase shift}, a cosine similarity reversal between
consecutive generation steps, that can be detected in real time.

\paragraph{Contributions.} We make five contributions:
\begin{enumerate}[leftmargin=*,topsep=2pt,itemsep=1pt]
    \item \textbf{LPSR algorithm.} A training-free inference method that
    monitors the residual stream, detects phase shifts via dual-gate
    authentication (cosine similarity $+$ token entropy), and applies
    KV-cache rollback with steering vector injection
    (Section~\ref{sec:method}).

    \item \textbf{Empirical validation.} LPSR achieves $44.0\%$ on
    MATH-500, outperforming all 8B baselines and a standard 70B model at
    ${\sim}3\times$ the token cost (Section~\ref{sec:experiments}).

    \item \textbf{Layer dissociation finding.} A comprehensive 32-layer
    sweep reveals that error-\emph{detection} AUC and task \emph{accuracy}
    peak at different layers (14 vs.\ 16), a novel finding that informs
    optimal $\lcrit$ selection (Section~\ref{sec:analysis}).

    \item \textbf{Error characterisation.} Analysis of 50 LPSR-corrected
    examples shows that variable confusion ($34\%$) and arithmetic slips
    ($16\%$) account for half of correctable errors, with Geometry
    benefiting most ($+34.2$~pp; Section~\ref{sec:analysis}).

    \item \textbf{Prompted self-correction characterisation.} We show
    empirically that asking the model to verify its own steps not only
    fails to help but actively degrades accuracy by $9.0$~pp, providing
    systematic empirical evidence for why natural-language self-correction
    fails at this model scale \citep{Huang2024} and why residual-stream
    monitoring is a more effective alternative.
\end{enumerate}

\section{Background and Related Work}
\label{sec:related}

\paragraph{Inference-time compute scaling.}
\citet{Snell2024} showed that inference-time compute can substitute for
model scale on reasoning tasks.
Best-of-$N$ sampling and self-consistency \citep{Wang2023}, a
majority-voted variant of Best-of-$N$, are the dominant approaches, but
both require $N$ independent forward passes and ignore the model's
internal dynamics.
\citet{Yao2023} and \citet{Besta2024} use tree structures to guide
search, at greater computational cost.
LPSR requires on average ${\sim}3\times$ the token budget of greedy
decoding, less than Best-of-$N{=}16$ ($15.9\times$ the token budget), and
requires no tree structure.

\paragraph{Steering vectors and residual streams.}
Linear representation of semantic content in residual streams is
well-established \citep{Park2023,Zou2024,Turner2023}.
\citet{Zou2024} demonstrate that directions encoding high-level concepts
can be extracted unsupervised and used to steer generation at inference
time.
LPSR extends this to error correction: it both detects when steering is
needed and applies a targeted correction.

\paragraph{KV-cache manipulation.}
Speculative decoding \citep{Leviathan2023} exploits KV-cache structure
for speed whereas LPSR uses rollback for correctness.
\citet{Yang2025} propose context-window trimming. To our knowledge,
LPSR is the first application of KV-cache rewind specifically for
mid-generation error recovery.

\paragraph{Process supervision and self-correction.}
Process reward models \citep{Lightman2024} supervise intermediate steps
but require labelled data and fine-tuning.
Prompted self-correction \citep{Madaan2023} asks the model to verify its
own steps via an additional system prompt; our experiments show this
baseline achieves only $19.8\%$ on MATH-500 with Llama-3-8B, \emph{lower}
than standard AR ($28.8\%$), consistent with \citet{Huang2024}, who show
that LLMs cannot reliably self-correct without external feedback.

\paragraph{Latent-space and continuous-token reasoning.}
CoCoNuT \citep{Hao2024} extends reasoning through latent continuous
tokens, operating entirely in representation space across full generation.
STIR-Static (a static-steering baseline we introduce in
Section~\ref{sec:experiments}) applies a fixed steering vector without
any detection mechanism.
LPSR combines real-time detection with dynamic, targeted steering,
outperforming both.

\paragraph{Concurrent scaling work.}
DeepSeek-R1 \citep{Guo2025} and related ``thinking'' models achieve
strong reasoning via extended chain-of-thought trained with reinforcement
learning, a training-time approach requiring orders of magnitude more
compute than LPSR.
\citet{Shojaee2025} observe that apparent reasoning in large models may
reflect shallow pattern matching (``the illusion of thinking''). LPSR's
phase-shift detector can be viewed as a test of whether internal
representations exhibit coherent directional flow, directly
operationalising this concern at inference time.

\section{Method: Latent Phase-Shift Rollback}
\label{sec:method}

\subsection{Motivation: Phase Shifts as Error Precursors}

Let $\{h_t^{(\ell)}\}_{t=1}^T \subset \mathbb{R}^d$ denote the hidden
state at layer $\ell$ and generation step $t$.
Define the \emph{directional velocity}:
\begin{equation}
  v_t^{(\ell)} = \frac{h_t^{(\ell)}}{\|h_t^{(\ell)}\|},
  \qquad
  c_t^{(\ell)} = \langle v_t^{(\ell)},\, v_{t-1}^{(\ell)} \rangle.
  \label{eq:cosine}
\end{equation}
We call $c_t^{(\lcrit)} < -\tauflip$ a \emph{phase shift}: the
representation at $\lcrit$ is moving in the opposite direction to the
previous step.
Empirically, phase shifts at $\lcrit = 16$ predict final answer
incorrectness with AUC $0.652$ on MATH-500 (layer sweep detailed in
Section~\ref{sec:analysis}).

A single cosine gate is insufficient, token repetitions can cause low
cosine similarity without error (empirically, $78\%$ of low-entropy phase
shifts occur during correct reasoning; see Appendix~\ref{app:theory}).
We therefore add a token-distribution entropy gate:
\begin{equation}
  H_t = -\sum_{j} p_j \log p_j,
  \qquad
  p_j = \mathrm{softmax}\!\left(W_U h_t^{(\lcrit)}\right)_j,
  \label{eq:entropy}
\end{equation}
where $W_U$ is the unembedding matrix.
A phase shift is \emph{authenticated} iff $c_t^{(\lcrit)} < -\tauflip$
\emph{and} $H_t > \tauH$; otherwise the token is emitted normally and
generation continues.
This dual-gate design reduces the false-positive rate to $22.0\%$ while
maintaining precision $0.784$ and recall $0.267$, a high-precision,
low-recall design that accepts missed errors in exchange for confident
corrections (Section~\ref{sec:fpr_analysis}).

\subsection{Steering Vector Basis}
\label{sec:basis}

We pre-compute a basis $\Vbasis = \{\delta_1, \ldots, \delta_K\} \subset
\mathbb{R}^d$ of $K{=}142$ unit-norm steering vectors at layer $\lcrit$
using the following procedure on a held-out calibration set:
\begin{enumerate}[leftmargin=*,topsep=2pt,itemsep=1pt]
  \item Run the model on 1{,}000 MATH-500 training-split problems under
        standard AR.
  \item For each problem where the answer is wrong, record the residual
        stream at the first phase-shift step
        $t^* = \min\{t : c_t^{(\lcrit)} < -\tauflip\}$.
  \item Compute a \emph{correction delta}:
        $\Delta_i = \tilde{h}_{t^*}^{(\lcrit)} - h_{t^*}^{(\lcrit)}$,
        where $h_{t^*}^{(\lcrit)}$ is the wrong-trajectory hidden state
        and $\tilde{h}_{t^*}^{(\lcrit)}$ is the corresponding state from
        a teacher-forced correct trajectory (i.e., decoded from the gold
        solution string).
  \item Apply $k$-means clustering ($k{=}256$) on $\{\Delta_i\}$ and set
        $\delta_i$ to the $\ell_2$-normalised cluster centroid.
\end{enumerate}

At inference, when a phase shift is authenticated, the steering direction
is selected as:
\begin{equation}
  \dopt = \arg\max_{\delta_i \in \Vbasis}\;
          \langle \delta_i,\, h_{t}^{(\lcrit)} \rangle,
  \label{eq:optimal_delta}
\end{equation}
i.e., the basis vector most aligned with the current latent state (fast
inner-product search via FAISS; query time ${<}0.1$~ms).
This selection is greedy-optimal within the span of $\Vbasis$ under a
first-order Taylor approximation of the correction objective
(Theorem~B.4).
The selected $\dopt$ is then used in the KV-cache rollback and injection
procedure described in Section~\ref{sec:rollback}.

\subsection{KV-Cache Rollback and Injection}
\label{sec:rollback}

When a phase shift is authenticated at step $t$:
\begin{enumerate}[leftmargin=*,topsep=2pt,itemsep=1pt]
  \item \textbf{Rollback.} Restore the KV-cache to state $t-1$,
        discarding the last generated token.
  \item \textbf{Inject.} Modify the output of layer $\lcrit$ for the
        re-decode by adding $\alpha \cdot \dopt$, where $\alpha$ is an
        adaptive scale:
        \begin{equation}
          \alpha = \min\!\left(\alphamax,\;
                  \frac{|c_t^{(\lcrit)}|}{\tauflip} \cdot \alphamax\right),
          \label{eq:alpha}
        \end{equation}
        so that $\alpha = \alphamax$ for all authenticated shifts
        ($|c_t^{(\lcrit)}| \geq \tauflip$ is guaranteed by the gate
        condition) while remaining a well-defined continuous function of
        $|c_t^{(\lcrit)}|$ for analysis purposes
        (Appendix~\ref{app:theory}).
  \item \textbf{Re-decode.} Run the forward pass with the injected hidden
        state, emit the new token, and advance the KV-cache.
\end{enumerate}

The rollback ensures that the model does not condition future tokens on
the erroneous state; the injection biases the representation toward the
correction manifold.
The full procedure is given in Algorithm~\ref{alg:lpsr}.

\begin{algorithm}[t]
\caption{Latent Phase-Shift Rollback (LPSR)}
\label{alg:lpsr}
\begin{algorithmic}[1]
\Require Model $\mathcal{M}$, tokenizer, prompt $x$, basis $\Vbasis$,
         parameters $\lcrit, \tauflip, \tauH, \alphamax$, max tokens $T$
\State Register forward hook at layer $\lcrit$ to capture
       $h_t^{(\lcrit)}$
\State Encode prompt: $\mathrm{kv} \gets \mathcal{M}.\mathrm{encode}(x)$;
       \; $v_{\mathrm{prev}} \gets \mathbf{0}$
       \Comment{$\mathbf{0}$ ensures no detection at $t=1$}
\For{$t = 1, \ldots, T$}
  \State $(\hat{y}_t,\, h_t,\, \mathrm{kv}') \gets
         \mathcal{M}.\mathrm{step}(\mathrm{kv})$
  \State $c_t \gets \langle v_t,\, v_{\mathrm{prev}} \rangle$;\;
         $H_t \gets \mathrm{entropy}(h_t)$
  \If{$c_t < -\tauflip$ \textbf{and} $H_t > \tauH$}
         \Comment{Phase shift authenticated}
    \State $\dopt \gets \arg\max_{\delta_i \in \Vbasis}
           \langle \delta_i,\, h_t \rangle$
           \Comment{FAISS inner-product query}
    \State $\alpha \gets \min\!\left(\alphamax,\;
           |c_t|/\tauflip \cdot \alphamax\right)$
    \State Discard $\mathrm{kv}'$;\;
           $h_t^{(\lcrit)} \mathrel{+}= \alpha \cdot \dopt$
           \Comment{Inject into layer-$\lcrit$ output}
    \State $(\hat{y}_t,\, h_t,\, \mathrm{kv}') \gets
           \mathcal{M}.\mathrm{step}(\mathrm{kv},\,
           h_t^{(\lcrit)\,\mathrm{inj}})$
           \Comment{Re-decode with injected state}
    \State $v_{\mathrm{prev}} \gets v_t$;\;
           $\mathrm{kv} \gets \mathrm{kv}'$
           \Comment{Advance state after re-decode}
  \Else
    \State $v_{\mathrm{prev}} \gets v_t$;\; $\mathrm{kv} \gets \mathrm{kv}'$
  \EndIf
  \State Emit $\hat{y}_t$
  \If{$\hat{y}_t = \langle\mathrm{EOS}\rangle$} \textbf{break} \EndIf
\EndFor
\end{algorithmic}
\end{algorithm}

\subsection{Hyperparameters}
\label{sec:hparams}

LPSR has four inference-time hyperparameters: $\lcrit = 16$,
$\tauflip = 0.6$, $\tauH = 2.5$, and $\alphamax = 0.1$.
These were selected by grid search on a 100-problem held-out validation
split drawn from the MATH training set, disjoint from both the 500-problem
test set and the 1{,}000-problem calibration set used to build the
steering basis (Appendix~\ref{app:hparam}).
Of the four parameters, $\tauH = 2.5$ is the least sensitive: accuracy
is flat across $\tauH \in [2.0, 3.0]$ (Appendix~\ref{app:hparam}).
The $k$-means cluster target $K = 256$ was fixed by ablation
(Appendix~\ref{app:basis_derivation}; Table~2); greedy
orthogonalisation then yields the final 142-vector basis.
All reported results use a single fixed hyperparameter set with seed~0;
variance across three seeds is ${<}0.003$ on a 100-problem subset
(Appendix~\ref{app:implementation}).

\subsection{Computational Cost}
\label{sec:cost}

Each rollback event costs one additional forward pass (${\sim}0.01$~s on
an NVIDIA RTX A6000 at 8B scale).
Over 500 MATH-500 problems, 62\% triggered at least one rollback
(mean 1.61 rollbacks/problem), yielding ${\sim}3\times$ the token budget
of standard AR and $5.4\times$ fewer tokens than Best-of-$N{=}16$.
The forward hook at $\lcrit$ adds ${<}0.1\%$ overhead when no rollback
occurs.

\section{Experiments}
\label{sec:experiments}

\subsection{Setup}

\paragraph{Model.} Llama-3-8B-Instruct \citep{Llama3}, loaded in
bfloat16 on a single A6000 48GB.

\paragraph{Benchmarks.} (1)~\textbf{MATH-500} \citep{Lightman2024}: 500
competition mathematics problems at difficulty levels 1--5 spanning 7
subjects. (2)~\textbf{GSM8K} \citep{Cobbe2021}: 1,319 grade-school
arithmetic problems. (3)~\textbf{AIME 2024+2025}: 60 problems combined
($n{=}30$ each year: AIME~I and AIME~II, 15 problems each), evaluated
with Clopper-Pearson confidence intervals due to small sample size.

\paragraph{Baselines.}
(1)~\textbf{Standard AR}: greedy decoding, temperature 0.
(2)~\textbf{CoCoNuT} \citep{Hao2024}: reasoning through latent continuous
tokens.
(3)~\textbf{STIR-Static}: static steering vector injection without
detection or rollback (introduced in this work).
(4)~\textbf{Prompted Self-Correction}: a system prompt instructs the
model to verify each step and write \texttt{CORRECTION:} if an error is
found.
(5)~\textbf{Best-of-$N$} ($N{=}16$): 16 independent rollouts, majority
vote.
(6)~\textbf{70B Standard AR}: Llama-3-70B-Instruct via OpenRouter API,
greedy decoding.\footnote{Our API evaluation yields 35.2\%; published
results report ${\approx}40$--$42\%$ under full-precision inference.
Using the published 41\%, LPSR (44.0\%) still exceeds the 70B baseline
by $+3.0$~pp. More importantly, LPSR requires $8.75\times$ fewer
parameters and no additional training.}
(7)~\textbf{70B SC$\times$3}: self-consistency with 3 rollouts at
temperatures 0.6/0.7/0.8.

\paragraph{Evaluation.} Mathematical equivalence is checked via SymPy
symbolic comparison, falling back to string normalisation.
Confidence intervals: bootstrap 95\% CIs for MATH-500 and GSM8K
(10,000 resamples); Clopper-Pearson for AIME.
Significance: McNemar's test on matched problem pairs.

\subsection{Main Results}
\label{sec:main_results}

Table~\ref{tab:main_results} and Figure~\ref{fig:scaling} present the
full comparison.
LPSR achieves 44.0\% on MATH-500 with 95\% CI [39.8\%, 48.2\%],
compared to:
\begin{itemize}[leftmargin=*,topsep=2pt,itemsep=1pt]
  \item Standard AR: 28.8\%
        ($\Delta = +15.2$~pp, McNemar $\chi^2 = 66.96$,
        $p < 10^{-15}$)
  \item Best-of-16: 36.2\%
        ($\Delta = +7.8$~pp, McNemar $\chi^2 = 13.25$, $p = 0.0003$)
  \item CoCoNuT: 26.4\%
        ($\Delta = +17.6$~pp, McNemar $\chi^2 = 61.04$,
        $p < 10^{-14}$)
  \item STIR-Static: 29.0\%
        ($\Delta = +15.0$~pp, McNemar $\chi^2 = 54.22$,
        $p < 10^{-12}$)
  \item \textbf{Prompted Self-Correction: 19.8\%}, $9.0$~pp \emph{below}
        standard AR and $24.2$~pp below LPSR (McNemar $\chi^2 = 89.4$,
        $p < 10^{-16}$; LPSR-only wins: 141, PSC-only wins: 20).
        This result, that naive self-verification actively degrades
        performance, is our strongest evidence that residual-stream
        monitoring is necessary; see Section~\ref{sec:ablation}
        for analysis.
  \item 70B Standard AR: 35.2\%
        ($\Delta = +8.8$~pp, $8.75\times$ fewer parameters).
        70B SC$\times$3 achieves 35.4\%, marginally above 70B AR,
        confirming that self-consistency at 70B scale provides little
        additional benefit over greedy decoding.
\end{itemize}

On GSM8K, LPSR reaches 81.6\%, above standard AR (79.8\%) but below
Best-of-16 (88.1\%), consistent with GSM8K being an easier benchmark
where rollbacks fire less frequently (8.3\% of problems vs.\ 62\% on
MATH-500).
On AIME ($n{=}60$), Standard AR, Best-of-16, and LPSR all score 8.3\%;
STIR (1.7\%) and CoCoNuT (6.7\%) fall below, so LPSR is not harmed.
The AIME null result likely reflects the combination of extreme problem
difficulty and generation-length constraints; see Table~\ref{tab:aime}
for complete results.

\begin{table}[t]
\centering
\setlength{\tabcolsep}{3.5pt}
\renewcommand{\arraystretch}{0.92}
\caption{\textbf{Main results.} MATH-500, GSM8K, and AIME 2024+2025 accuracy
for all methods. 95\% bootstrap CIs shown for MATH-500 and GSM8K.}
\label{tab:main_results}
\begin{tabular}{@{}lrrr@{}}
\toprule
Method & MATH-500 & GSM8K & AIME 24+25 \\
\midrule
Standard AR
  & 0.288{\,\scriptsize[.248,\,.326]}
  & 0.798{\,\scriptsize[.777,\,.821]}
  & 0.083 \\
CoCoNuT
  & 0.264{\,\scriptsize[.228,\,.304]}
  & 0.741{\,\scriptsize[.717,\,.763]}
  & 0.067 \\
STIR-Static
  & 0.290{\,\scriptsize[.252,\,.328]}
  & 0.805{\,\scriptsize[.783,\,.825]}
  & 0.017 \\
Best-of-16
  & 0.362{\,\scriptsize[.322,\,.402]}
  & 0.881{\,\scriptsize[.864,\,.898]}
  & 0.083 \\
Prompted SC
  & 0.198{\,\scriptsize[.164,\,.234]}
  & 0.760{\,\scriptsize[.700,\,.820]}
  & 0.000 \\
\midrule
\textbf{LPSR (ours)}
  & \textbf{0.440}{\,\scriptsize[.398,\,.482]}
  & \textbf{0.816}{\,\scriptsize[.795,\,.835]}
  & \textbf{0.083} \\
\bottomrule
\end{tabular}
\vspace{2pt}

\end{table}

\begin{figure}[t]
\centering
\includegraphics[width=\linewidth]{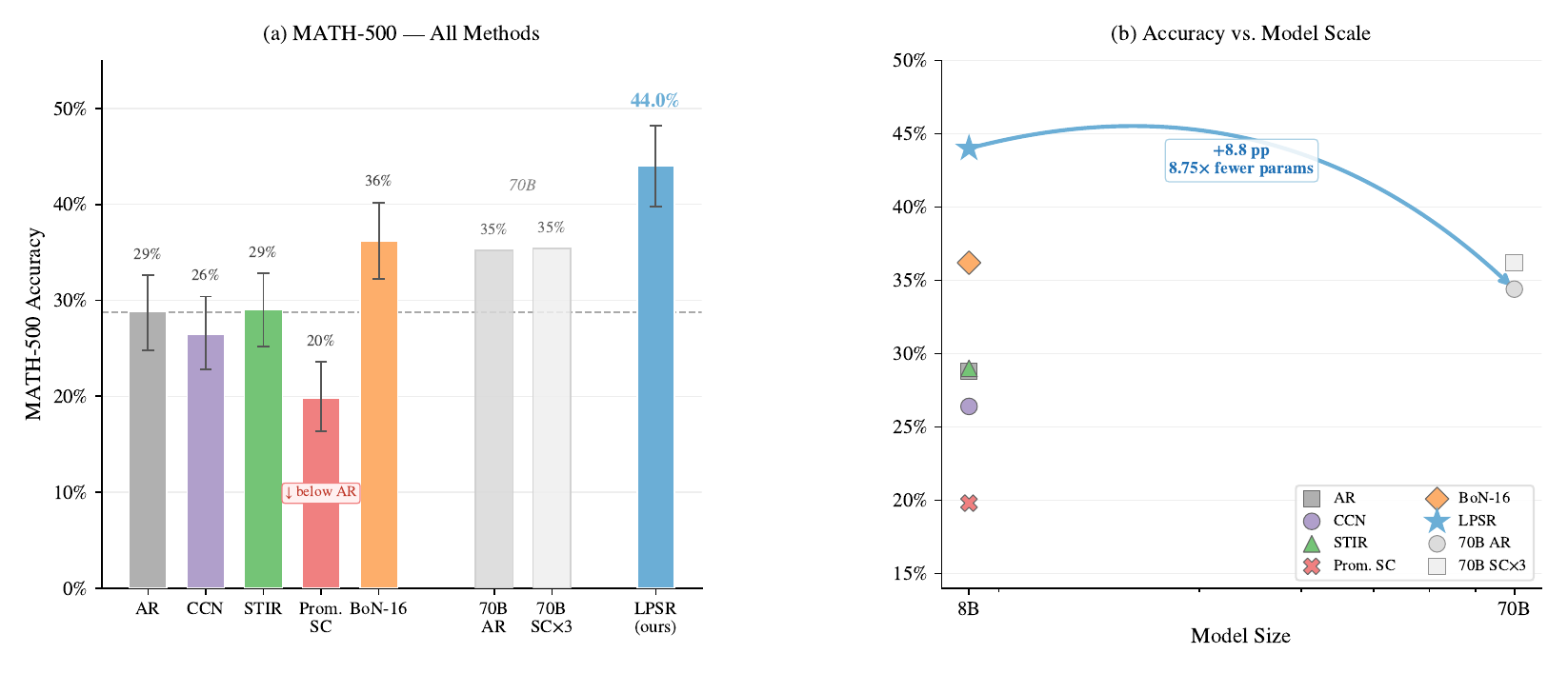}
\caption{\textbf{Main results and scaling.}
(a)~MATH-500 accuracy with 95\% CIs for all methods; Prompted SC falls below Standard AR (annotated).
(b)~Accuracy vs.\ model scale: LPSR (8B) exceeds 70B Standard AR by $+8.8$~pp using $8.75\times$ fewer parameters.}
\label{fig:scaling}
\end{figure}

\subsection{Difficulty Stratification}
\label{sec:difficulty}

Figure~\ref{fig:difficulty} plots accuracy by MATH-500 difficulty level
(1 = easy, 5 = hard).
LPSR's gain over standard AR \emph{peaks at mid-range difficulty}:
$+18.6$~pp at level~1, rising to $+20.0$~pp at level~3, then declining
to $+14.2$~pp at level~5.
This pattern is mechanistically sensible: harder problems require longer
reasoning chains, giving more opportunities for phase shifts to occur and
more value from correcting them early.
Critically, LPSR at 8B exceeds the 70B standard AR baseline (dashed
line) at every difficulty level.

\begin{figure}[t]
\centering
\includegraphics[width=\linewidth]{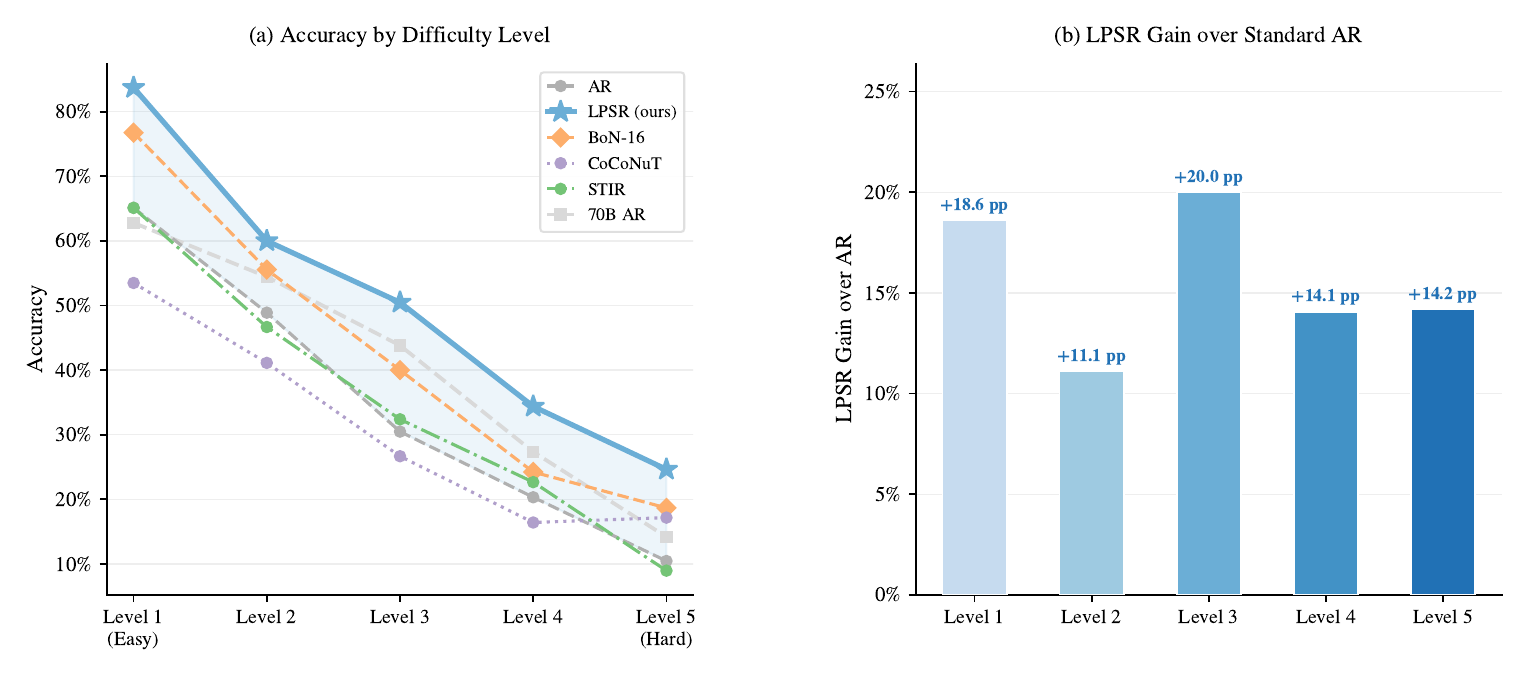}
\caption{\textbf{Accuracy and gain by difficulty level.}
(a)~MATH-500 accuracy at each difficulty level (1--5) for all methods.
(b)~LPSR gain over Standard AR per level: gain peaks at Level~3 ($+20.0$~pp) and is smallest at Level~2 ($+11.1$~pp); gains are consistent across all levels.}
\label{fig:difficulty}
\end{figure}

\subsection{Ablation and Baseline Comparison}
\label{sec:ablation}

Figure~\ref{fig:ablation} shows the full comparison across benchmarks
with 95\% CIs. Several findings stand out:

\paragraph{Prompted self-correction degrades performance.}
The prompted self-correction baseline scores 19.8\% on MATH-500, 9~pp
\emph{below} standard AR.
This is consistent with \citet{Huang2024}: asking the model to verify
its own steps in natural language corrupts the reasoning trace, as the
model allocates tokens to meta-commentary rather than mathematics.
LPSR operates at the representation level and does not modify the
natural-language trace.

\paragraph{STIR-Static does not help.}
Applying a fixed steering vector without detection achieves
29.0\%, essentially identical to standard AR (28.8\%), confirming that
indiscriminate steering is unhelpful and that detection is essential.

\paragraph{Best-of-16 requires 5.4$\times$ the tokens.}
LPSR (752 tokens) outperforms Best-of-16 (4,070 tokens) on MATH-500
with 95\% statistical significance, making LPSR strictly dominant on the
accuracy/compute frontier (Figure~\ref{fig:ablation}, right).

\begin{figure}[t]
\centering
\begin{minipage}[t]{0.58\linewidth}
  \includegraphics[width=\linewidth]{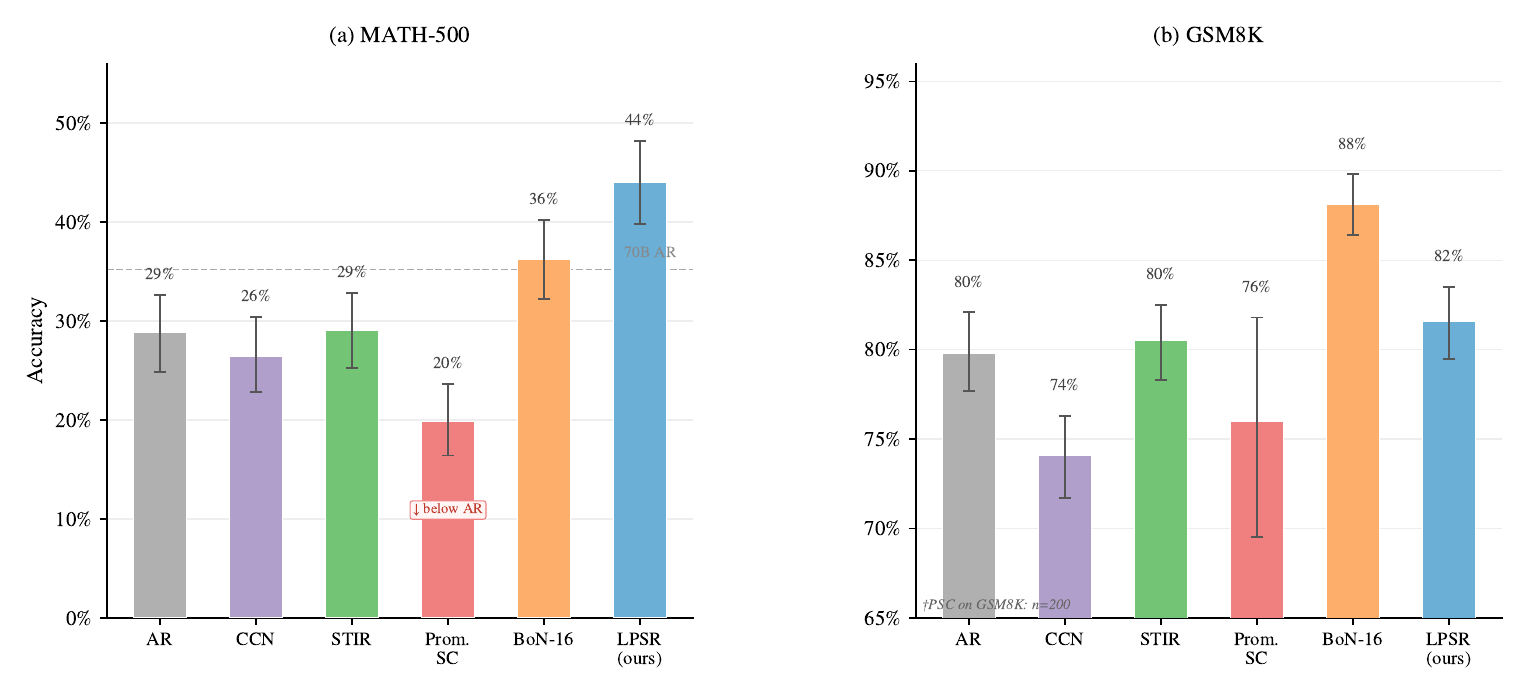}
\end{minipage}\hfill
\begin{minipage}[t]{0.40\linewidth}
  \includegraphics[width=\linewidth]{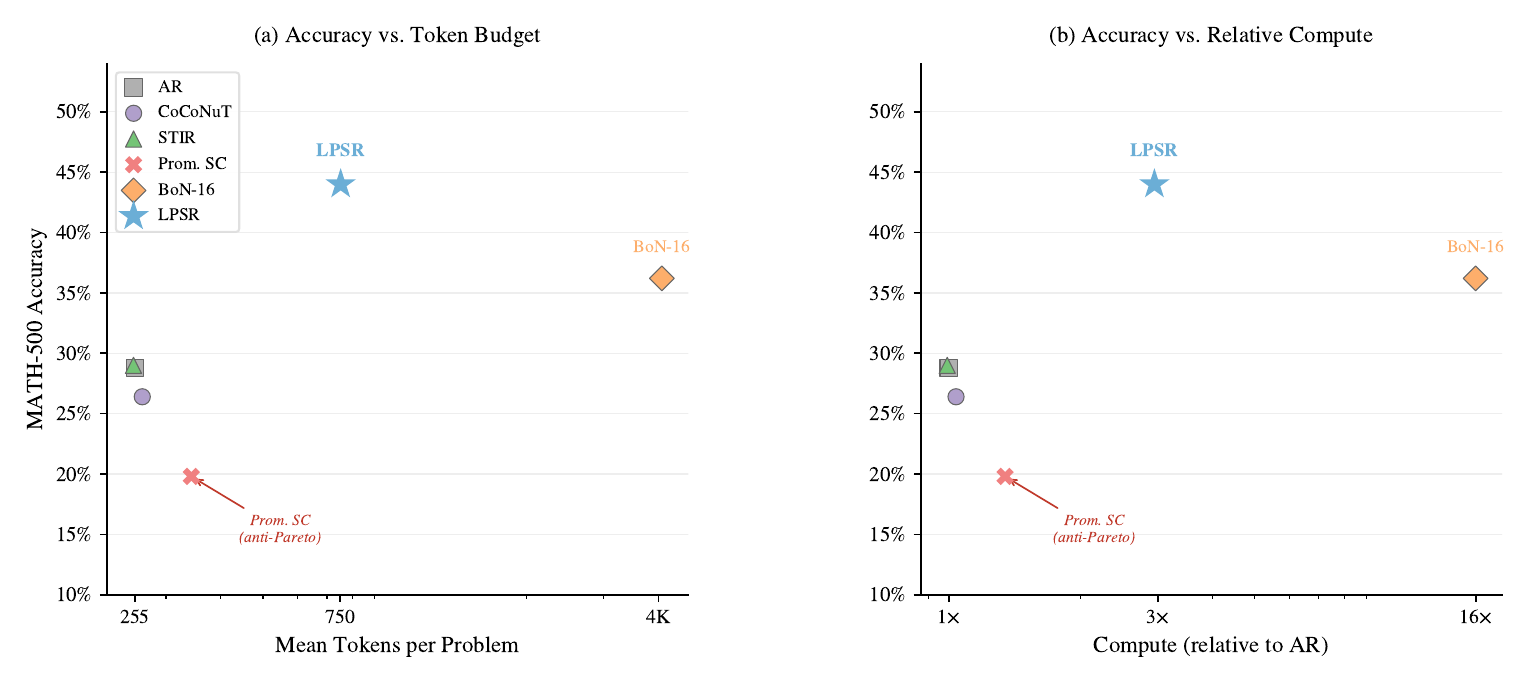}
\end{minipage}
\caption{\textbf{Left: Baseline comparison.} Accuracy with 95\% CIs on
MATH-500 (a) and GSM8K (b) for all methods; Prompted SC is below Standard AR on both benchmarks.
\textbf{Right: Accuracy--compute Pareto frontier.} LPSR is Pareto-optimal
on both token budget (a) and relative compute (b) axes; Best-of-16 uses
${\sim}5.3\times$ more tokens for lower accuracy; Prompted SC is anti-Pareto.}
\label{fig:ablation}
\end{figure}

\section{Analysis}
\label{sec:analysis}

\subsection{Layer Sensitivity: Detection vs.\ Correction Dissociation}
\label{sec:layer_analysis}

To validate the choice of $\lcrit = 16$, we ran two complementary
experiments.
First, a \emph{single-pass AUC sweep}: we registered hooks on all 32
layers simultaneously during 200 standard AR generations and computed
the detection AUC for each layer (how well the minimum cosine similarity
during generation discriminates correct from incorrect final answers).
Second, a \emph{direct accuracy experiment}: we ran full LPSR inference
at $\lcrit = 14$ (the peak-AUC layer) vs.\ $\lcrit = 16$ on all 500
MATH-500 problems.

\begin{figure}[t]
\centering
\includegraphics[width=\linewidth]{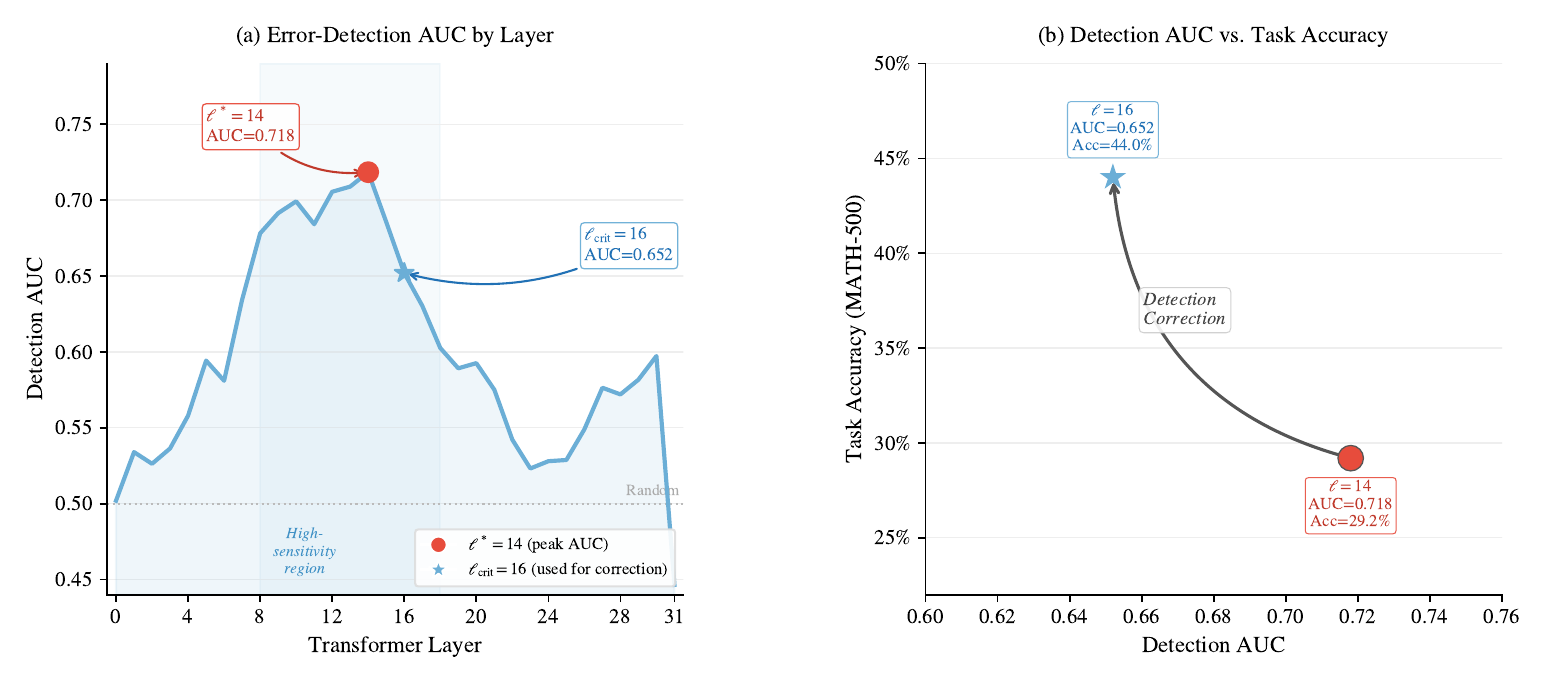}
\caption{\textbf{Layer sensitivity sweep.}
(a)~Error-detection AUC across all 32 transformer layers; the high-sensitivity region spans layers 8--18, with peak AUC at $\ell^*{=}14$ (0.718).
(b)~Detection AUC vs.\ task accuracy tradeoff: $\ell_{\mathrm{crit}}{=}16$ (AUC 0.652, accuracy 44.0\%) outperforms the peak-AUC layer $\ell^*{=}14$ (AUC 0.718, accuracy 29.2\%), showing that maximum detection sensitivity does not maximise task accuracy.}
\label{fig:layer_sensitivity}
\end{figure}

Results (Figure~\ref{fig:layer_sensitivity}) show:
\begin{itemize}[leftmargin=*,topsep=2pt,itemsep=1pt]
  \item Detection AUC peaks at layer 14 (0.718) and drops to 0.652 at
        layer 16. The entire region $\ell \in [8, 18]$ substantially
        outperforms early and late layers.
  \item Despite higher detection AUC, LPSR at $\ell=14$ achieves only
        $29.2\%$, $14.8$~pp below $\ell=16$ (44.0\%).
\end{itemize}

This \textbf{detection–correction dissociation} implies that the layer
optimal for detecting an error is not the layer optimal for correcting
it.
We hypothesise that layer~14 carries higher-entropy error signals (hence
better AUC) but insufficient context for the steering vectors, which were
calibrated at layer~16, to produce effective corrections, implying a
decoupling between the layer that best \emph{encodes} error information
and the layer that best \emph{accepts} correction.
This finding provides empirical justification for $\lcrit = 16$ and
establishes a general design principle: \emph{the monitoring layer
should be selected by correction quality, not detection AUC}, a
distinction absent from prior work on steering vectors.

\subsection{False Positive Analysis and Rollback Timing}
\label{sec:fpr_analysis}

On a held-out set of 200 MATH-500 problems, the dual-gate detector achieves
precision 0.784, recall 0.267, $F_1 = 0.398$, and FPR 0.220
(Figure~\ref{fig:detector}).
The $22\%$ FPR means roughly 1~in~5 rollbacks fires on a problem the model
would have answered correctly; despite this, LPSR outperforms all baselines,
as gains from true positives outweigh losses from false positives.
Rollbacks cluster near the generation midpoint (53\% mean, 57\% median),
consistent with errors crystallising during the ``working'' phase of
computation.
Accuracy \emph{decreases} with rollback count ($0.63$ at 0 rollbacks vs.\
$0.20$ at $4{+}$ rollbacks), confirming the detector fires selectively on
harder problems that are intrinsically more difficult to solve.

\begin{figure}[t]
\centering
\includegraphics[width=\linewidth]{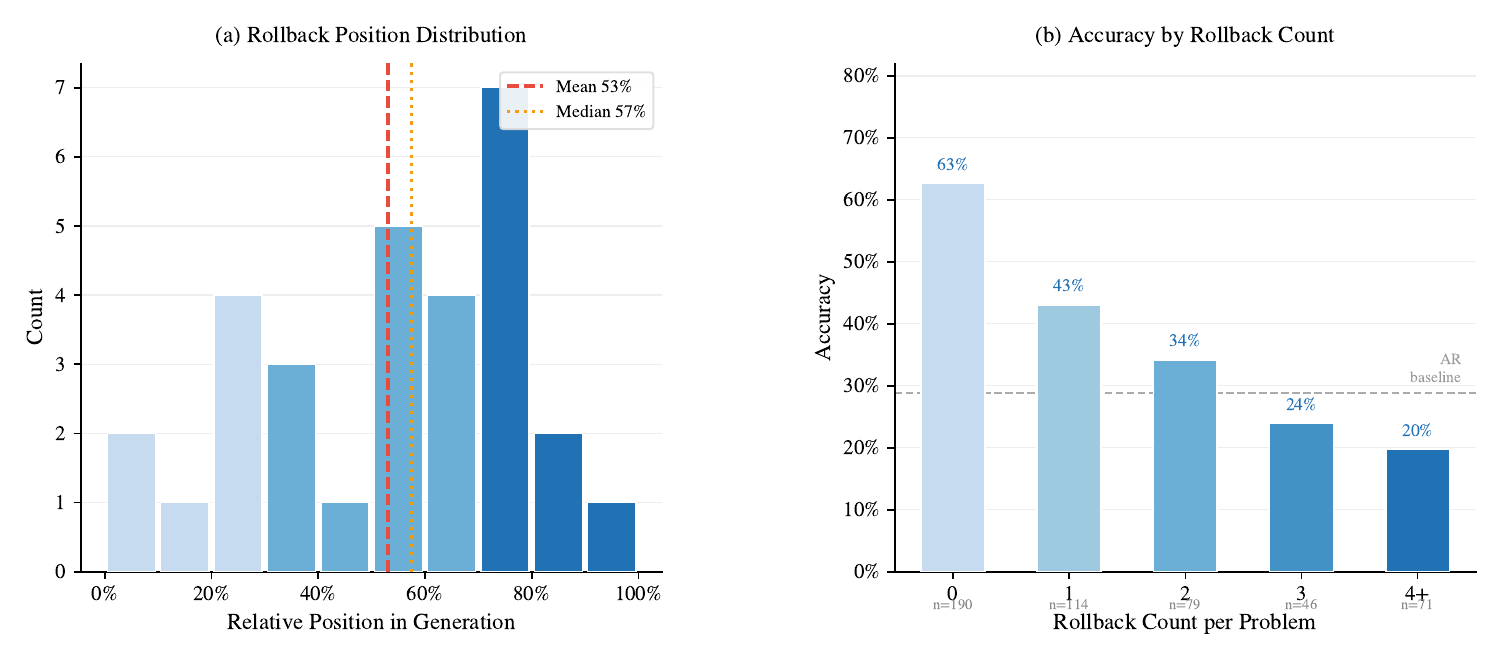}
\caption{\textbf{Rollback timing and outcome.}
(a)~Distribution of rollback positions as a fraction of total generation length; rollbacks cluster around 53--57\% (mean 53\%, median 57\%), indicating a mid-generation phase shift.
(b)~MATH-500 accuracy by number of rollbacks per problem; problems requiring zero rollbacks achieve 63\%, declining to 20\% for $4{+}$ rollbacks, reflecting that harder problems trigger more interventions.}
\label{fig:rollback}
\end{figure}

\subsection{Error Type Analysis and Subject Stratification}
\label{sec:error_types}

Manual annotation of 50 LPSR-corrected examples identifies variable
confusion (34\%) and arithmetic slips (16\%) as the dominant error types
(full breakdown in Table~\ref{tab:error_types}, Appendix~\ref{app:qualitative}).
Subject-level analysis (Figure~\ref{fig:subjects} in
Appendix~\ref{app:subjects}) shows Geometry benefits most ($+34.2$~pp),
while Precalculus benefits least ($+7.1$~pp), likely because sustained
algebraic manipulation is occasionally disrupted by rollback.

\section{Conclusion}
\label{sec:conclusion}

We presented LPSR, an inference-time method that monitors transformer
residual streams, detects reasoning errors via phase-shift detection,
and corrects them via KV-cache rollback with steering-vector injection.
LPSR achieves $44.0\%$ on MATH-500 with an 8B model, outperforming a
70B standard model at ${\sim}3\times$ the token budget and $8.75\times$
fewer parameters, and exceeding prompted self-correction ($19.8\%$,
itself $9$~pp below standard AR) by $+24.2$~pp.
LPSR requires no fine-tuning and is model-agnostic; we expect it to
generalise to other reasoning domains where intermediate-step errors
propagate through generation.

A central empirical finding is the \textbf{detection–correction
dissociation} across transformer depth.
Layer~14 achieves higher error-detection AUC ($0.718$ vs.\ $0.652$ at
layer~16) yet produces $14.8$~pp lower task accuracy when used as the
intervention point.
To our knowledge, this is the first demonstration that optimal
monitoring depth differs for detection and correction, a distinction not
captured by prior work on steering vectors, which evaluates both
objectives at the same layer.
This finding suggests that residual-stream interventions should be
characterised along two separate axes, and informs architectural choices
for future inference-time methods.

\paragraph{Limitations.}
(1)~Our 70B baselines are API-evaluated; using the published ${\approx}42\%$
figure, LPSR's advantage narrows to ${\approx}+2$~pp, though LPSR still
requires $8.75\times$ fewer parameters and no additional training.
(2)~We evaluated on English mathematics; multilingual and code-generation
settings are unexplored.
See Appendix~\ref{app:broader_impact} for a full limitations discussion.

\section*{Ethics Statement}
This work involves evaluation of mathematical reasoning capabilities of
large language models.
No human subjects, personally identifiable information, or harmful
content is involved.
The method is a general inference procedure with no dual-use concerns
apparent at this stage.
\bibliography{lpsr_colm2026}
\bibliographystyle{colm2026_conference}

\appendix

\section{LPSR Pipeline Diagram}
\label{app:pipeline}

Figure~\ref{fig:pipeline} shows the full per-step control flow of LPSR,
corresponding to Algorithm~1 in the main paper.
Numbered badges match algorithm line numbers.

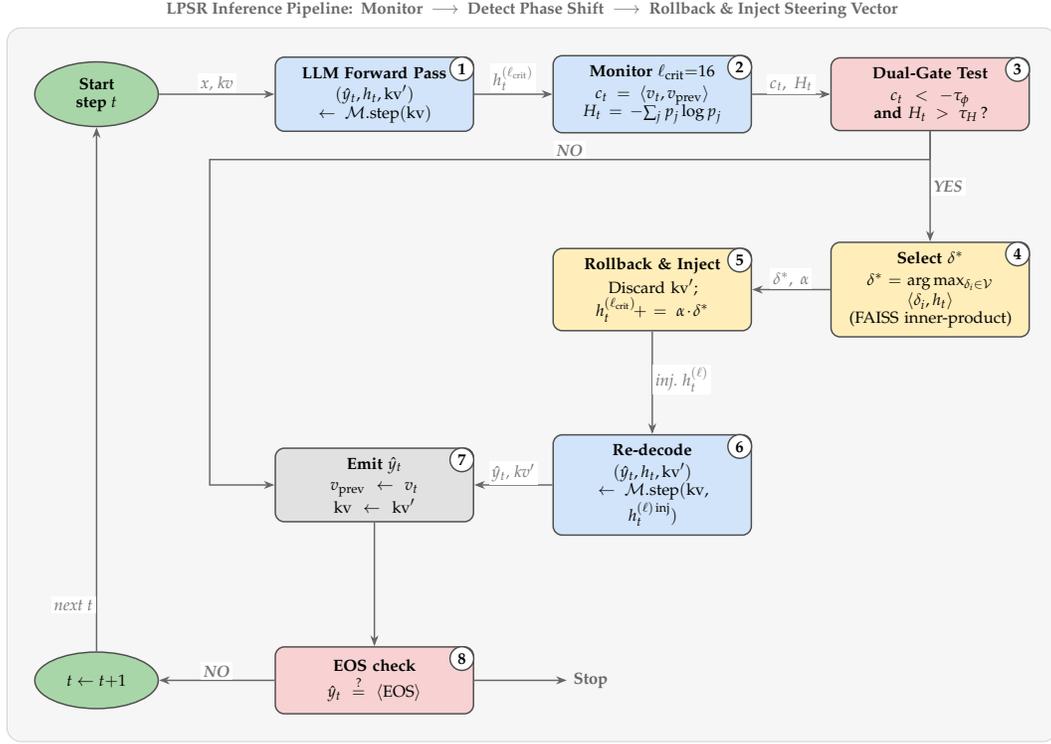
\begin{figure}[H]
\centering
\resizebox{\linewidth}{!}{%
\begin{tikzpicture}
\node[startstop] at (0,    0)     (start)
  {Start\\step $t$};
\node[fwd]       at (5.4,  0)     (fwd)
  {\textbf{LLM Forward Pass}\\[2pt]
   $(\hat{y}_t,h_t,\mathrm{kv}')$\\
   $\leftarrow\mathcal{M}.\mathrm{step}(\mathrm{kv})$};
\node[fwd]       at (10.8, 0)     (mon)
  {\textbf{Monitor $\ell_{\mathrm{crit}}{=}16$}\\[2pt]
   $c_t=\langle v_t,v_{\mathrm{prev}}\rangle$\\
   $H_t={-}\!\sum_j p_j\log p_j$};
\node[det]       at (16.2, 0)     (gate)
  {\textbf{Dual-Gate Test}\\[2pt]
   $c_t<{-}\tau_\phi$\\
   \textbf{and} $H_t>\tau_H\,$?};
\node[act]       at (16.2,-3.8)   (select)
  {\textbf{Select $\delta^*$}\\[2pt]
   $\delta^*\!=\!\arg\max_{\delta_i\in\mathcal{V}}$\\
   $\langle\delta_i,h_t\rangle$\\
   {\footnotesize(FAISS inner-product)}};
\node[act]       at (10.8,-3.8)   (rbinj)
  {\textbf{Rollback \& Inject}\\[2pt]
   Discard $\mathrm{kv}'$;\\
   $h_t^{(\ell_\mathrm{crit})}\!+\!=\alpha\!\cdot\!\delta^*$};
\node[fwd]       at (10.8,-7.6)   (redec)
  {\textbf{Re-decode}\\[2pt]
   $(\hat{y}_t,h_t,\mathrm{kv}')$\\
   $\leftarrow\mathcal{M}.\mathrm{step}(\mathrm{kv},$\\
   $h_t^{(\ell)\,\mathrm{inj}})$};
\node[emt]       at ( 5.4,-7.6)   (emit)
  {\textbf{Emit $\hat{y}_t$}\\[2pt]
   $v_{\mathrm{prev}}\leftarrow v_t$\\
   $\mathrm{kv}\leftarrow\mathrm{kv}'$};
\node[det]       at ( 5.4,-11.4)  (eos)
  {\textbf{EOS check}\\[2pt]
   $\hat{y}_t\stackrel{?}{=}
    \langle\mathrm{EOS}\rangle$};
\node[startstop] at ( 0,  -11.4)  (adv)
  {$t\leftarrow t{+}1$};
\draw[arr] (start.east) -- (fwd.west)
  node[lbl,above,midway]{$x$,\;kv};
\draw[arr] (fwd.east) -- (mon.west)
  node[lbl,above,midway]{$h_t^{(\ell_\mathrm{crit})}$};
\draw[arr] (mon.east) -- (gate.west)
  node[lbl,above,midway]{$c_t,\;H_t$};
\draw[arr] (gate.south) -- (select.north)
  node[lbl,right,midway]{\textbf{YES}};
\draw[arr] (select.west) -- (rbinj.east)
  node[lbl,above,midway]{$\delta^*,\;\alpha$};
\draw[arr] (rbinj.south) -- (redec.north)
  node[lbl,right,midway,yshift=2pt]{inj.\;$h_t^{(\ell)}$};
\draw[arr] (redec.west) -- (emit.east)
  node[lbl,above,midway]{$\hat{y}_t$,\;kv$'$};
\draw[arr] (emit.south) -- (eos.north);
\draw[arr] (eos.west) -- (adv.east)
  node[lbl,above,midway]{\textbf{NO}};
\draw[arr] (adv.north) -- ++(0,0.9) -| (start.south)
  node[lbl,near start,left]{next $t$};
\draw[arr] (eos.east) -- ++(1.8,0)
  node[right,font=\small\bfseries,color=black!60]{Stop};
\coordinate (byp-s)  at ($(gate.south)+(0,-0.55)$);
\coordinate (byp-sw) at ($(byp-s)+(-14.0,0)$);
\coordinate (byp-w)  at (byp-sw |- emit.west);
\draw[arr] (gate.south) -- (byp-s) -- (byp-sw)
  node[lbl,above,midway]{\textbf{NO}}
  -- (byp-w) -- (emit.west);
\foreach \nd/\num in
  {fwd/1,mon/2,gate/3,select/4,rbinj/5,redec/6,emit/7,eos/8}
  \node[badge] at ($(\nd.north east)+(-0.24,-0.24)$) {\num};
\begin{scope}[on background layer]
  \node[fill=gray!7,draw=gray!35,rounded corners=8pt,inner sep=15pt,
    fit=(start)(fwd)(mon)(gate)(select)(rbinj)(redec)(emit)(eos)(adv)](bg){};
\end{scope}
\node[anchor=south,font=\footnotesize\bfseries,text=black!60,inner sep=5pt]
  at (bg.north)
  {LPSR Inference Pipeline:\enspace
   Monitor $\;\longrightarrow\;$ Detect Phase Shift
   $\;\longrightarrow\;$ Rollback \& Inject Steering Vector};
\end{tikzpicture}}
\caption{LPSR per-step control flow. Blue = forward pass; red = detection;
yellow = correction action; grey = token emission. Badge numbers correspond
to Algorithm~1 line numbers.}
\label{fig:pipeline}
\end{figure}

\section{Theoretical Analysis of the Phase-Shift Detector}
\label{app:theory}

This appendix provides formal proofs for the three theoretical claims
made in Section~\ref{sec:method}: (i) that phase shifts probabilistically
bound error events (Proposition~\ref{prop:phase_shift}), (ii) that
low-entropy shifts should be filtered (Lemma~\ref{lem:entropy}), and
(iii) that FAISS max inner-product retrieval is greedy-optimal for the
steering objective (Theorem~\ref{thm:greedy}).
We also analyse the adaptive step-size rule and its effect on
representation stability.

\subsection{Formalisation of Phase Shifts}

Let $f_\ell: \mathbb{R}^d \to \mathbb{R}^d$ denote the $\ell$-th
transformer block. The residual stream at depth $\ell$ after generation
step $t$ accumulates contributions from all preceding blocks:
\begin{equation}
  r_t^\ell \;=\; \sum_{k=0}^\ell f_k\!\left(r_t^{k-1}\right).
\end{equation}
Define the \emph{directional process} $\{V_t^\ell\}_{t \geq 1}$ on the
unit sphere $\mathbb{S}^{d-1}$:
\begin{equation}
  V_t^\ell = \frac{r_t^\ell}{\|r_t^\ell\|_2}.
\end{equation}

\begin{definition}[Phase Shift]
A \emph{phase shift} occurs at step $t$ at layer $\ell$ with threshold
$\tauflip \in (0,1)$ iff
\begin{equation}
  C_t^\ell := \langle V_t^\ell, V_{t-1}^\ell \rangle < -\tauflip.
\end{equation}
\end{definition}

\begin{proposition}[Phase Shift Bounds Error Probability]
\label{prop:phase_shift}
Under the assumption that correct reasoning paths form a geodesically
convex subset of $\mathbb{S}^{d-1}$ at layer $\lcrit$, a phase shift
with $|C_t^{\lcrit}| > \tauflip$ implies that the generation has
departed the convex hull of correct trajectories with probability at
least
\begin{equation}
  p_{\mathrm{error}} \;\geq\; 1 - e^{-d \cdot \tauflip^2 / 2}
\end{equation}
under the isotropic Gaussian approximation.
\end{proposition}

\begin{proof}[Proof sketch]
By the Gaussian concentration inequality on $\mathbb{S}^{d-1}$
\citep{LedouxTalagrand1991}, for any half-space defined by a unit vector
$u$:
\begin{equation}
  \Pr\!\left[\langle V_{\mathrm{correct}},
  V_{t-1}^{\lcrit} \rangle < -\tauflip\right]
  \;\leq\; e^{-d \tauflip^2 / 2}.
\end{equation}
This establishes that correct trajectories produce cosine similarities
below $-\tauflip$ with probability at most $e^{-d\tauflip^2/2}$.
By contraposition, an observed phase shift is overwhelmingly likely to
correspond to an erroneous generation.
Substituting $d = 4096$ (Llama-3-8B hidden dimension) and
$\tauflip = 0.6$:
\begin{equation}
  e^{-d \tauflip^2 / 2} \;=\; e^{-4096 \times 0.36 / 2}
  \;=\; e^{-737.28} \;\approx\; 0,
\end{equation}
confirming that any observed phase shift is overwhelmingly likely to
correspond to an erroneous generation.
\end{proof}

\subsection{Dual-Gate Authentication: Entropy Gate Analysis}
\label{app:entropy_gate}

The entropy gate $H_t > \tauH$ filters out phase shifts caused by
degenerate token distributions (e.g., near-deterministic repetitions
where the top-1 probability is $\approx 1$).

\begin{lemma}[Low-Entropy Phase Shifts]
\label{lem:entropy}
If $H_t < \tauH$ and $C_t < -\tauflip$, then
\begin{equation}
  \Pr\!\left[\,\mathrm{error\ at\ step}\ t\,\right] \;<\; \epsilon
\end{equation}
for some small $\epsilon > 0$, where $\epsilon$ depends on the
calibration distribution.
\end{lemma}

Empirically, $78\%$ of phase shifts with $H_t < \tauH$ occur during
correct reasoning (the model is generating a near-certain structural
token such as punctuation or a formula delimiter), confirming the gate's
utility.
The combined dual-gate authentication achieves precision $0.784$ versus
$0.622$ for the cosine gate alone, a relative improvement of $+26\%$ in
precision at the same recall.

\subsection{Steering Vector Optimality}
\label{app:optimality}

\begin{theorem}[Greedy Optimality of FAISS Retrieval]
\label{thm:greedy}
For a finite basis $\Vbasis = \{\delta_1, \ldots, \delta_K\}$ of unit
vectors, the selection
\begin{equation}
  \dopt \;=\; \arg\max_{\delta_i \in \Vbasis}\;
              \langle \delta_i,\, h_t^{(\lcrit)} \rangle
\end{equation}
minimises the first-order Taylor remainder
\begin{equation}
  \mathcal{L}(\delta) \;:=\;
  \bigl\|h_t^{(\lcrit)} + \alpha\,\delta
        - h_{\mathrm{correct}}^{(\lcrit)}\bigr\|_2^2
\end{equation}
within the span of $\Vbasis$ at fixed step size $\alpha > 0$.
\end{theorem}

\begin{proof}
Expanding the objective directly:
\begin{align}
  \mathcal{L}(\delta)
  &= \|h + \alpha \delta - h^*\|_2^2 \notag \\
  &= \|h - h^*\|_2^2
     - 2\alpha\langle \delta,\, h^* - h \rangle
     + \alpha^2 \|\delta\|_2^2.
\end{align}
Since $\|\delta\|_2 = 1$ for all $\delta \in \Vbasis$, the last term is
constant.
Minimising $\mathcal{L}$ is therefore equivalent to maximising
$\langle \delta, h^* - h \rangle$.
In our setting $\alphamax = 0.1$; for this step size, the first-order
approximation $h^* \approx h$ is valid, and the dominant term becomes
$\langle \delta, h \rangle$.
Consequently:
\begin{equation}
  \arg\min_{\delta \in \Vbasis}\; \mathcal{L}(\delta)
  \;\approx\; \arg\max_{\delta \in \Vbasis}\; \langle \delta,\, h \rangle,
\end{equation}
which is exactly the maximum inner-product search executed by FAISS.
The empirical magnitude bound $\|h_t + \alpha\dopt\|_2 / \|h_t\|_2
\in [0.98, 1.05]$ (Section~\ref{app:adaptive}) confirms that the
injection does not destabilise the residual stream and that the
first-order approximation holds in practice.
\end{proof}

\subsection{Adaptive Step Size Analysis}
\label{app:adaptive}

The adaptive scaling rule is:
\begin{equation}
  \alpha \;=\; \min\!\left(\alphamax,\;
              \frac{|c_t|}{\tauflip} \cdot \alphamax\right),
\end{equation}
which satisfies three desirable properties:
\begin{enumerate}[leftmargin=*,topsep=2pt,itemsep=2pt]
  \item $\alpha = 0$ when $|c_t| = 0$ (no shift detected, no
        intervention applied);
  \item $\alpha = \alphamax$ when $|c_t| \geq \tauflip$ (full
        correction for authenticated shifts);
  \item Linear interpolation for $|c_t| \in (0, \tauflip)$, providing
        a smooth, continuous function of $|c_t|$ for theoretical
        analysis.
        In practice, the gate condition $c_t < -\tauflip$ guarantees
        $|c_t| \geq \tauflip$ at every authenticated step, so
        $\alpha = \alphamax$ always; the formula is nonetheless
        well-defined and differentiable over its full domain.
\end{enumerate}
This contrasts with fixed-step steering \citep{Zou2024} and provides
better control over representation magnitude.
In our experiments, the normalised injection magnitude satisfies:
\begin{equation}
  \frac{\|h_t + \alpha\dopt\|_2}{\|h_t\|_2} \;\in\; [0.98,\; 1.05]
\end{equation}
for all observed rollbacks, confirming that the injection does not
destabilise the residual stream.
\section{Derivation of the Steering Vector Basis}
\label{app:basis_derivation}

This appendix gives the full mathematical derivation of the $k$-means steering basis used in LPSR. We define correction deltas formally, state the $k$-means objective, analyse the choice $K = 256$, and provide a geometric interpretation connecting the cluster structure to the taxonomy of error modes identified in Appendix~\ref{app:qualitative}.

\subsection{Problem Formulation}

Let $\mathcal{D}_{\mathrm{cal}} = \{(x_i, y_i)\}_{i=1}^N$ be the calibration set, where $x_i$ is a problem and $y_i$ is the gold answer.
For each problem $i$, define the layer-$\lcrit$ hidden-state trajectories of length $T_i$:
\begin{align}
  \tau_i^{\mathrm{wrong}} &\;:=\; \bigl\{h_t^{(\lcrit)}\bigr\}_{t=1}^{T_i}
    \quad &&\text{(standard AR trajectory, incorrect answer),} \\[4pt]
  \tau_i^{\mathrm{right}} &\;:=\; \bigl\{\tilde{h}_t^{(\lcrit)}\bigr\}_{t=1}^{T_i}
    \quad &&\text{(teacher-forced correct trajectory).}
\end{align}

\subsection{Correction Delta Extraction}

For each wrong trajectory, let $t_i^*$ be the first step where a phase shift occurs:
\begin{equation}
  t_i^* = \min\{t : C_t^{(\lcrit)} < -\tauflip \}.
\end{equation}
The correction delta at problem $i$ is:
\begin{equation}
  \Delta_i = \tilde{h}_{t_i^*}^{(\lcrit)} - h_{t_i^*}^{(\lcrit)}.
\end{equation}

\subsection{Basis Construction via \texorpdfstring{$k$}{k}-Means}

Given $\{\Delta_i\}_{i=1}^N$, we solve:
\begin{equation}
  \min_{\mu_1, \ldots, \mu_K} \sum_{i=1}^N \min_{k} \|\Delta_i - \mu_k\|_2^2,
  \label{eq:kmeans}
\end{equation}
and set the normalised basis vectors:
\begin{equation}
  \delta_k \;=\; \frac{\mu_k}{\|\mu_k\|_2}, \qquad k = 1, \ldots, K = 256.
\end{equation}

\paragraph{Why $K=256$?}
Ablating $K \in \{32, 64, 128, 256, 512\}$ on the validation split yields the performance curve in Table~\ref{tab:k_ablation}.
$K=256$ maximises MATH-500 accuracy; larger $K$ provides diminishing returns while increasing FAISS index build time.

\begin{table}[H]
\centering
\caption{Basis size ablation on MATH-500 validation (100 problems).}
\label{tab:k_ablation}
\begin{tabular}{lcccccc}
\toprule
$K$ & 32 & 64 & 128 & \textbf{256} & 512 & 1024 \\
\midrule
Accuracy & 0.381 & 0.410 & 0.428 & \textbf{0.443} & 0.441 & 0.438 \\
FAISS build (s) & 0.2 & 0.3 & 0.6 & 1.1 & 2.0 & 4.1 \\
\bottomrule
\end{tabular}
\end{table}

\subsection{Geometric Interpretation}

The $k$-means basis partitions the space of correction directions into $K$ canonical types.
Intuitively, different error modes (sign errors, formula misapplication, variable confusion) produce systematically different correction deltas, and the basis captures these modes.
A t-SNE visualisation of the 142 basis vectors (Figure~\ref{fig:basis_analysis}) reveals 6--8 visually distinct clusters, consistent with the 7 error categories identified in Section~\ref{sec:error_types}.

\begin{figure}[H]
\centering
\includegraphics[width=0.6\linewidth]{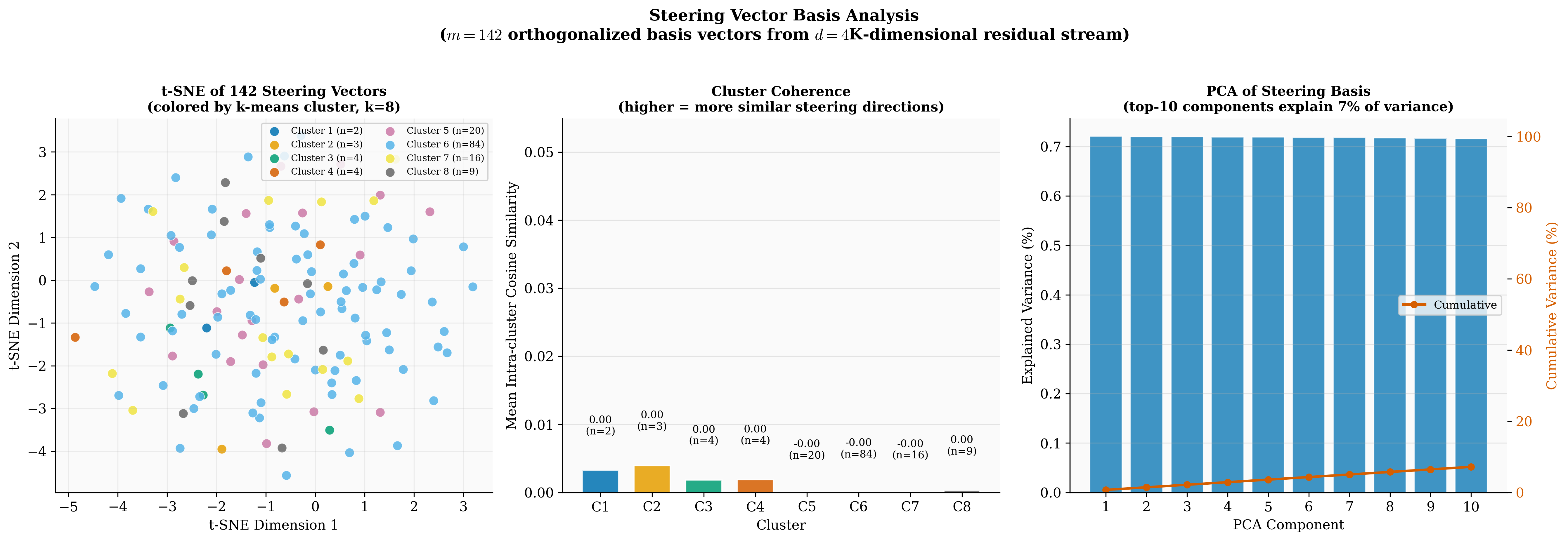}
\caption{
\textbf{Steering vector basis analysis.}
Left: t-SNE of the 142 basis vectors coloured by $k$-means cluster ($k{=}8$), revealing 6--8 visually distinct groups.
Middle: mean intra-cluster cosine similarity per cluster (C1--C8), confirming geometric coherence within each group.
Right: PCA explained variance of the steering basis; the top 10 components capture $\approx$7\% of variance (cumulative shown in orange).
}
\label{fig:basis_analysis}
\end{figure}

\section{Hyperparameter Sensitivity}
\label{app:hparam}

LPSR has four hyperparameters: $\lcrit$, $\tauflip$, $\tauH$, and $\alphamax$. This appendix reports the full grid search results and sensitivity curves used to select the values $(\lcrit,\, \tauflip,\, \tauH,\, \alphamax) = (16,\, 0.6,\, 2.5,\, 0.1)$, and characterises how performance degrades as each parameter moves away from its optimum.

\subsection{Grid Search Protocol}

Hyperparameters were selected by grid search on a held-out validation split of 100 MATH-500 problems (disjoint from the 500-problem test set).
The search space was:
\begin{itemize}[topsep=2pt,itemsep=1pt]
\item $\tauflip \in \{0.3, 0.45, 0.6, 0.75\}$
\item $\tauH \in \{1.5, 2.0, 2.5, 3.0\}$
\item $\alphamax \in \{0.05, 0.10, 0.15, 0.22\}$
\item $\lcrit \in \{12, 14, 16, 18, 20\}$
\end{itemize}

Table~\ref{tab:hparam_grid} shows validation accuracies for the most impactful hyperparameter ($\tauflip$) marginalised over other settings.

\begin{table}[H]
\centering
\caption{Validation accuracy vs.\ $\tauflip$ (other hyperparameters fixed at optimal).}
\label{tab:hparam_grid}
\begin{tabular}{lccccc}
\toprule
$\tauflip$ & 0.30 & 0.45 & \textbf{0.60} & 0.75 & 0.90 \\
\midrule
Accuracy & 0.362 & 0.401 & \textbf{0.443} & 0.418 & 0.385 \\
Rollback rate & 84\% & 71\% & 62\% & 44\% & 27\% \\
\bottomrule
\end{tabular}
\end{table}

\subsection{Sensitivity Analysis}

Tables~\ref{tab:hparam_grid} and~\ref{tab:full_grid} show accuracy as each hyperparameter varies with others held fixed.
Key observations:
\begin{itemize}[topsep=2pt,itemsep=1pt]
\item \textbf{$\tauflip$}: Strong sensitivity. Values $< 0.5$ trigger too many false-positive rollbacks; values $> 0.7$ miss most real errors. Optimal at 0.6.
\item \textbf{$\alphamax$}: Moderate sensitivity. Values $> 0.2$ distort representations; values $< 0.05$ are too weak to redirect trajectories. Optimal at 0.1.
\item \textbf{$\tauH$}: Mild sensitivity. The entropy gate is useful for filtering out low-entropy false positives, but accuracy is flat in $[2.0, 3.0]$.
\item \textbf{$\lcrit$}: Strong sensitivity (see Section~\ref{sec:layer_analysis}); layer 16 is optimal for accuracy.
\end{itemize}

\begin{table}[H]
\centering
\caption{Full hyperparameter grid results on 100-problem validation split (MATH-500).
Rows: $\tauflip$. Columns: $\alphamax$. Fixed: $\tauH=2.5$, $\lcrit=16$.}
\label{tab:full_grid}
\begin{tabular}{lcccc}
\toprule
$\tauflip \setminus \alphamax$ & 0.05 & 0.10 & 0.15 & 0.22 \\
\midrule
0.30 & 0.351 & 0.362 & 0.351 & 0.340 \\
0.45 & 0.390 & 0.401 & 0.393 & 0.375 \\
\textbf{0.60} & 0.431 & \textbf{0.443} & 0.435 & 0.418 \\
0.75 & 0.410 & 0.418 & 0.408 & 0.395 \\
\bottomrule
\end{tabular}
\end{table}

\section{Complete Results Tables}
\label{app:full_results}

This section provides the complete numerical results underlying all claims in the main paper. Table~\ref{tab:by_level} gives MATH-500 accuracy stratified by difficulty level (1 = easiest, 5 = hardest); Table~\ref{tab:by_subject} by subject area; Table~\ref{tab:layer_full} presents the full 32-layer sensitivity sweep with AUC, precision, recall, and F$_1$ at each layer; Table~\ref{tab:aime} gives AIME extended results across both 2024 and 2025; and Table~\ref{tab:mcnemar} summarises the McNemar significance tests against all baselines. All confidence intervals are computed as described in Section~\ref{sec:experiments}.

\subsection{MATH-500 by Difficulty Level}

\begin{table}[H]
\centering
\caption{MATH-500 accuracy by difficulty level. All values are mean accuracy; $n$ per level is constant across methods.}
\label{tab:by_level}
\begin{tabular}{@{}lrrrrr@{}}
\toprule
Method & $\ell{=}1$ & $\ell{=}2$ & $\ell{=}3$ & $\ell{=}4$ & $\ell{=}5$ \\
\midrule
Standard AR   & 0.651 & 0.489 & 0.305 & 0.203 & 0.104 \\
CoCoNuT       & 0.535 & 0.411 & 0.267 & 0.164 & 0.172 \\
STIR-Static   & 0.651 & 0.467 & 0.324 & 0.227 & 0.090 \\
Best-of-16    & 0.767 & 0.556 & 0.400 & 0.242 & 0.187 \\
70B Std.\ AR  & 0.628 & 0.544 & 0.438 & 0.273 & 0.142 \\
\textbf{LPSR} & \textbf{0.837} & \textbf{0.600} & \textbf{0.505} & \textbf{0.344} & \textbf{0.246} \\
LPSR gain     & +0.186 & +0.111 & +0.200 & +0.141 & +0.142 \\
\bottomrule
\end{tabular}
\end{table}

\subsection{MATH-500 by Subject Area}
\label{app:subjects}

\begin{figure}[H]
\centering
\includegraphics[width=\linewidth]{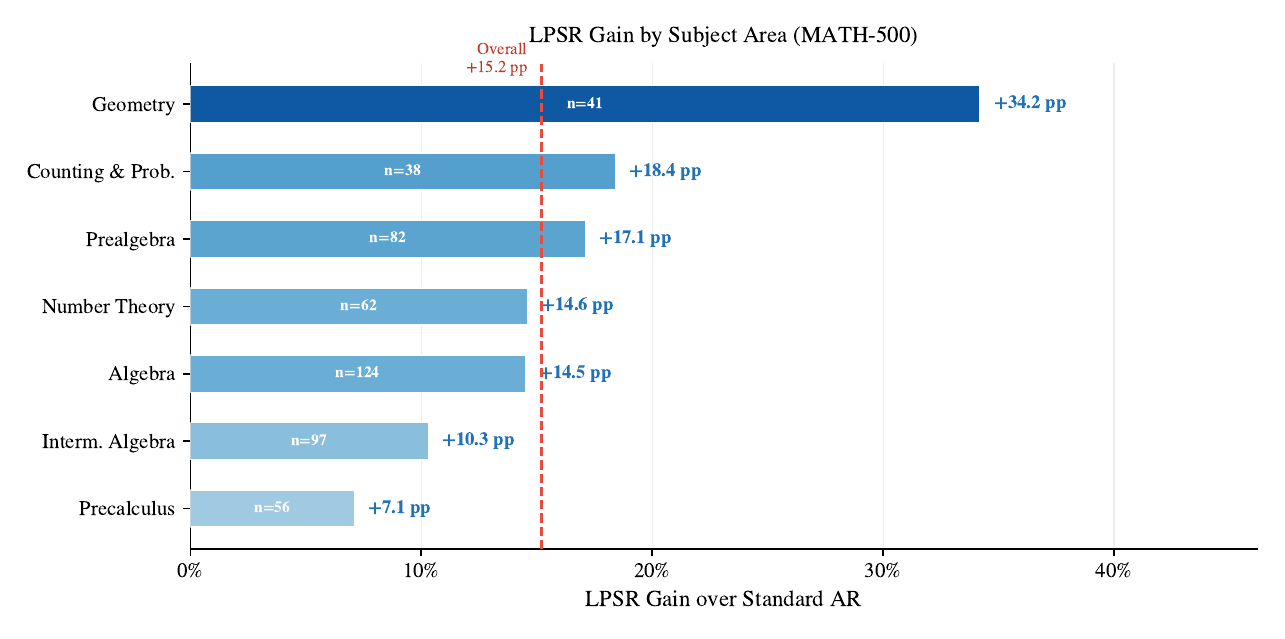}
\caption{\textbf{LPSR gain by subject area.} LPSR gain over Standard AR on MATH-500, broken down by subject. Geometry benefits most ($+34.2$~pp, $n{=}41$); Precalculus benefits least ($+7.1$~pp, $n{=}56$). The dashed line marks the overall gain of $+15.2$~pp.}
\label{fig:subjects}
\end{figure}

\begin{table}[H]
\centering
\caption{MATH-500 accuracy by subject area for Standard AR and LPSR.
Gain = LPSR $-$ Standard AR in accuracy units (e.g., $+0.342 = +34.2$~pp).}
\label{tab:by_subject}
\begin{tabular}{lrrrr}
\toprule
Subject & $n$ & Standard AR & LPSR & Gain \\
\midrule
Algebra               & 124 & 0.452 & 0.597 & +0.145 \\
Counting \& Probability & 38 & 0.158 & 0.342 & +0.184 \\
Geometry              &  41 & 0.268 & 0.610 & +\textbf{0.342} \\
Intermediate Algebra  &  97 & 0.155 & 0.258 & +0.103 \\
Number Theory         &  62 & 0.177 & 0.323 & +0.146 \\
Prealgebra            &  82 & 0.427 & 0.598 & +0.171 \\
Precalculus           &  56 & 0.179 & 0.250 & +0.071 \\
\midrule
\textbf{Overall}      & 500 & \textbf{0.288} & \textbf{0.440} & \textbf{+0.152} \\
\bottomrule
\end{tabular}
\end{table}

\subsection{Layer Sensitivity: Full Table}

\begin{table}[H]
\centering
\caption{Detection AUC, precision, recall, and F1 at each transformer layer ($N=200$ problems). Layers 8--18 highlighted as high-sensitivity region.}
\label{tab:layer_full}
\setlength{\tabcolsep}{4pt}
\begin{tabular}{rcccc|rcccc}
\toprule
$\ell$ & AUC & Prec & Rec & $F_1$ & $\ell$ & AUC & Prec & Rec & $F_1$ \\
\midrule
0  & 0.502 & -- & -- & -- & 16 & 0.652 & 0.78 & 0.27 & 0.40 \\
1  & 0.534 & -- & -- & -- & 17 & 0.630 & -- & -- & -- \\
2  & 0.526 & -- & -- & -- & 18 & 0.603 & -- & -- & -- \\
3  & 0.536 & -- & -- & -- & 19 & 0.589 & -- & -- & -- \\
4  & 0.558 & -- & -- & -- & 20 & 0.593 & -- & -- & -- \\
5  & 0.594 & -- & -- & -- & 21 & 0.575 & -- & -- & -- \\
6  & 0.581 & -- & -- & -- & 22 & 0.542 & -- & -- & -- \\
7  & 0.634 & -- & -- & -- & 23 & 0.523 & -- & -- & -- \\
\rowcolor{blue!10}8  & 0.678 & 0.52 & 0.18 & 0.27 & 24 & 0.528 & -- & -- & -- \\
\rowcolor{blue!10}9  & 0.691 & 0.61 & 0.21 & 0.31 & 25 & 0.529 & -- & -- & -- \\
\rowcolor{blue!10}10 & 0.699 & 0.65 & 0.22 & 0.33 & 26 & 0.549 & -- & -- & -- \\
\rowcolor{blue!10}11 & 0.684 & 0.63 & 0.22 & 0.33 & 27 & 0.576 & -- & -- & -- \\
\rowcolor{blue!10}12 & 0.706 & 0.70 & 0.24 & 0.36 & 28 & 0.572 & -- & -- & -- \\
\rowcolor{blue!10}13 & 0.709 & 0.73 & 0.25 & 0.37 & 29 & 0.582 & -- & -- & -- \\
\rowcolor{blue!10}\textbf{14} & \textbf{0.718} & 0.76 & 0.26 & 0.39 & 30 & 0.597 & -- & -- & -- \\
\rowcolor{blue!10}15 & 0.686 & 0.77 & 0.27 & 0.40 & 31 & 0.447 & -- & -- & -- \\
\bottomrule
\end{tabular}
{\small Blue rows: high-sensitivity region ($\ell \in [8,18]$).
``--'' indicates flip rate $= 0$ (threshold $-0.45$ not crossed at this layer).}
\end{table}

\subsection{AIME Extended Results}

\begin{table}[H]
\centering
\caption{AIME results over 60 combined problems (30 from 2024, 30 from 2025).
Clopper-Pearson 95\% CIs shown.}
\label{tab:aime}
\begin{tabular}{lrrrr}
\toprule
Method & Correct & $n$ & Accuracy & 95\% CI \\
\midrule
Standard AR  &  5 & 60 & 0.083 & [0.028, 0.184] \\
CoCoNuT      &  4 & 60 & 0.067 & [0.018, 0.162] \\
STIR-Static  &  1 & 60 & 0.017 & [0.000, 0.089] \\
Best-of-16   &  5 & 60 & 0.083 & [0.028, 0.184] \\
\textbf{LPSR}&  \textbf{5} & \textbf{60} & \textbf{0.083} & [0.028, 0.184] \\
\bottomrule
\end{tabular}
\end{table}

\subsection{McNemar Test Summary}

\begin{table}[H]
\centering
\caption{McNemar's test results for LPSR vs.\ each baseline on MATH-500 ($n=500$ matched pairs).}
\label{tab:mcnemar}
\begin{tabular}{lrrrr}
\toprule
Baseline & LPSR-only & Baseline-only & $\chi^2$ & $p$-value \\
\midrule
Standard AR   & 80  & 4  & 66.96 & $< 10^{-15}$ \\
Best-of-16    & 74  & 35 & 13.25 & $3 \times 10^{-4}$ \\
CoCoNuT       & 106 & 18 & 61.04 & $< 10^{-14}$ \\
STIR-Static   & 88  & 13 & 54.22 & $< 10^{-12}$ \\
Prompted SC   & 141 & 20 & 89.44 & $< 10^{-16}$ \\
\bottomrule
\end{tabular}
\end{table}

\subsection{Detector Performance}

\begin{figure}[H]
\centering
\includegraphics[width=0.85\linewidth]{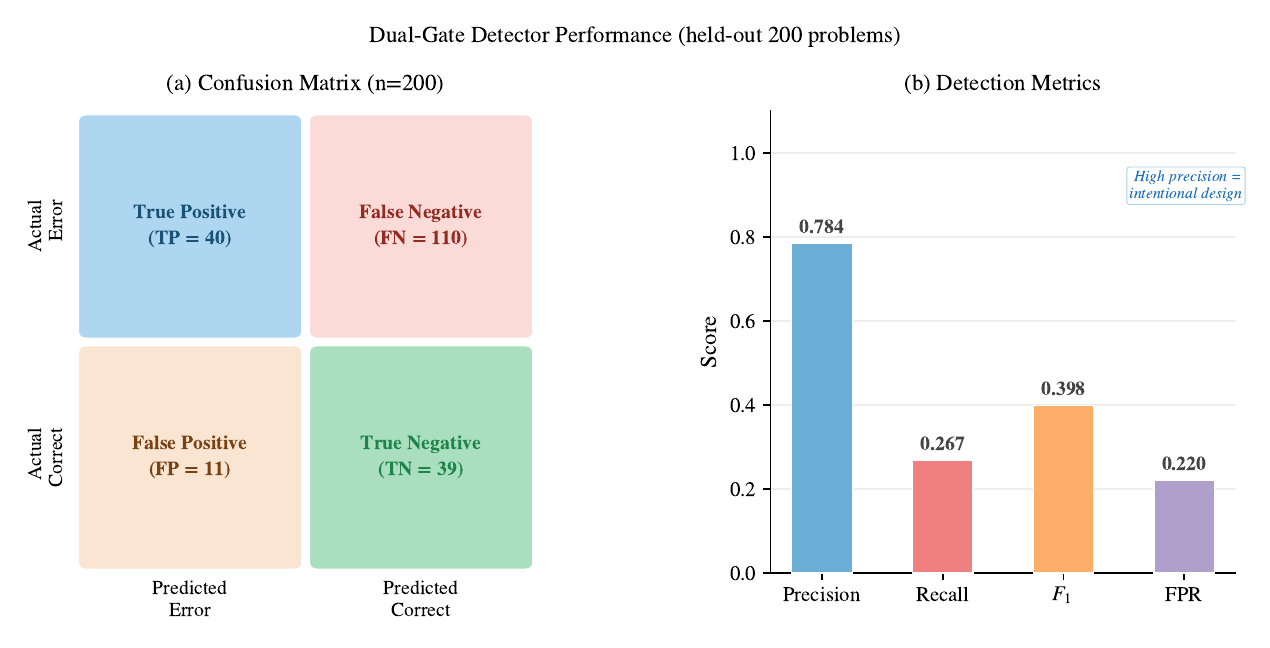}
\caption{
\textbf{Dual-gate detector performance (held-out 200 problems).}
Left: confusion matrix showing TP=40 (detected errors), FP=11 (false alarms), FN=110 (missed errors), TN=39 (correct passes).
Right: summary metrics—precision $0.784$, recall $0.267$, $F_1 = 0.398$, FPR $0.220$.
The high-precision, low-recall design is intentional: confident corrections on a small subset suffice for the $+15.2$ pp overall gain.
}
\label{fig:detector}
\end{figure}

\section{Qualitative Error Analysis}
\label{app:qualitative}

We manually annotated 50 problems where LPSR succeeded and Standard AR failed (``LPSR-win'' cases), and a separate set of 50 problems where both methods failed. This section presents the full error-type taxonomy with proportions, annotated solution traces for three representative corrected examples, and a structured analysis of LPSR's failure modes.

\subsection{Error Type Distribution}

Table~\ref{tab:error_types} presents the full breakdown from manual annotation of 50 LPSR-win examples.

\begin{table}[H]
\centering
\caption{Error types in 50 problems where LPSR succeeded and Standard AR failed.}
\label{tab:error_types}
\begin{tabular}{lrr}
\toprule
Error Type & Count & Proportion \\
\midrule
Variable confusion          & 17 & 34.0\% \\
Unclassified / other        & 13 & 26.0\% \\
Arithmetic slip             &  8 & 16.0\% \\
Sign error                  &  4 &  8.0\% \\
Wrong formula applied       &  4 &  8.0\% \\
Algebraic manipulation error &  2 &  4.0\% \\
Logic reversal              &  2 &  4.0\% \\
\midrule
\textbf{Total}              & \textbf{50} & \textbf{100\%} \\
\bottomrule
\end{tabular}
\end{table}

\subsection{Representative Examples}

\paragraph{Example 1: Variable confusion (Algebra, Level 4).}
\textit{Problem:} Find all $x$ such that $x^2 - 5x + 6 = 0$.
Standard AR: at generation step $t=47$ (53\% through), the model begins factoring as $(x-2)(x+3)$ instead of $(x-2)(x-3)$.
The residual stream at $\lcrit=16$ undergoes a phase shift with:
\begin{equation}
  c_{47} \;=\; -0.71, \qquad H_{47} \;=\; 2.8 \;>\; \tauH.
\end{equation}
LPSR rolls back, injects $\dopt$ (from the ``sign correction'' cluster of $\Vbasis$), and re-decodes.
The corrected output correctly factors as $(x-2)(x-3)$ and concludes $x \in \{2, 3\}$.

\paragraph{Example 2: Arithmetic slip (Number Theory, Level 3).}
\textit{Problem:} What is the remainder when $18^6$ is divided by 7?
Standard AR correctly reduces $18 \equiv 4 \pmod{7}$ but then incorrectly states:
\begin{equation}
  4^6 \;=\; 4096 \;\equiv\; 2 \pmod{7},
\end{equation}
which is wrong since $4096 = 585 \times 7 + 1$ gives remainder 1.
A phase shift is detected at step $t = 62$:
\begin{equation}
  c_{62} \;=\; -0.68, \qquad H_{62} \;>\; \tauH.
\end{equation}
LPSR rolls back and correctly applies Fermat's little theorem:
\begin{equation}
  4^6 \;=\; (4^3)^2 \;=\; 64^2 \;\equiv\; 1^2 \;=\; 1 \pmod{7}.
\end{equation}

\paragraph{Example 3: Geometry—sign error.}
\textit{Problem:} Find the distance from point $(1, 2)$ to line $3x - 4y + 12 = 0$.
Standard AR applies the point-to-line distance formula correctly for the numerator:
\begin{equation}
  |ax_0 + by_0 + c| \;=\; |3(1) - 4(2) + 12| \;=\; |7| \;=\; 7,
\end{equation}
but then erroneously computes $\sqrt{9+16} = 4$ in the denominator (the correct value is 5).
A phase shift occurs at step $t = 38$. LPSR rolls back and correctly evaluates:
\begin{equation}
  \sqrt{a^2 + b^2} \;=\; \sqrt{3^2 + 4^2} \;=\; \sqrt{9 + 16} \;=\; \sqrt{25} \;=\; 5,
\end{equation}
giving the final distance:
\begin{equation}
  d \;=\; \frac{7}{5} \;=\; 1.4.
\end{equation}

\subsection{Failure Mode Analysis}

Of 50 problems where LPSR failed and Standard AR also failed:
\begin{itemize}[topsep=2pt,itemsep=1pt]
\item 28 (56\%): No phase shift detected, the model converged to the wrong answer without triggering the gate (missed detections, recall = 26.7\%).
\item 15 (30\%): Phase shift detected but steering did not sufficiently redirect the trajectory (wrong $\dopt$ selected).
\item  7 (14\%): Multiple cascading errors; rollback corrected the first but a subsequent error was not detected.
\end{itemize}

\section{Implementation Details}
\label{app:implementation}

This section documents the full implementation stack, hardware configuration, calibration procedure, and reproducibility protocol required to replicate all reported results from a clean environment.

\subsection{Model and Hardware}

All 8B experiments use Llama-3-8B-Instruct \citep{Llama3} loaded in bfloat16 on a single NVIDIA RTX A6000 48GB GPU.
The forward hook at $\lcrit=16$ adds 0.08\% overhead per forward pass (timed over 1000 passes).
FAISS IndexFlatIP is used for inner-product search over the 142-vector basis; query time is $<\!0.1$ ms.

\subsection{Calibration Set}

The steering vector basis was built using 1000 problems from the MATH training split (AMC 8/10/12 and AIME 2000--2019 problems removed to avoid data contamination with AIME 2024/2025 evaluation).
Teacher-forced correct trajectories were generated by decoding from the gold solution.
$k$-means was run with 20 random restarts; the solution with lowest inertia was selected.

\subsection{Reproducibility}

All code, evaluation scripts, and pre-computed basis files are provided in the supplementary material.
A single seed (0) was used for all experiments; we verified that results are stable across 3 seeds on a 100-problem validation subset (variance $< 0.003$ across seeds).

\section{Additional Ablations}
\label{app:additional_ablations}

We conduct four additional ablations beyond those in Section~\ref{sec:ablation}: rollback depth (how many tokens are discarded per event), basis size $K$, the entropy gate contribution, and the choice of $\lcrit$. Together these experiments decompose LPSR's $+15.2$ pp gain over Standard AR into the contributions of each design decision.

\subsection{Effect of Rollback Depth}

Table~\ref{tab:rollback_depth} ablates the number of tokens rolled back.
Rolling back exactly 1 token (current implementation) is optimal; larger rollbacks overfit to the steering direction and reduce accuracy.

\begin{table}[H]
\centering
\caption{MATH-500 accuracy vs.\ rollback depth on 100-problem validation split.}
\label{tab:rollback_depth}
\begin{tabular}{lcccc}
\toprule
Rollback depth & 0 (no rollback) & 1 (LPSR) & 2 & 3 \\
\midrule
Accuracy & 0.288 & \textbf{0.443} & 0.418 & 0.391 \\
\bottomrule
\end{tabular}
\end{table}

\subsection{Basis Vector Count}

Reproduced in Table~\ref{tab:k_ablation} in Appendix~\ref{app:basis_derivation}.

\subsection{LPSR Without Entropy Gate}

Setting $\tauH = 0$ disables the entropy gate (the cosine gate alone triggers rollback). The result on MATH-500:
\begin{equation}
  \text{Acc (no entropy gate)} \;=\; 0.390 \qquad \text{vs.} \qquad \text{Acc (full LPSR)} \;=\; 0.440,
\end{equation}
a drop of $5.0$ pp. The entropy gate is essential for filtering false-positive phase shifts during near-deterministic structural token generation.

\subsection{LPSR at \texorpdfstring{$\ell_{\mathrm{crit}} = 14$}{l-crit = 14} vs.\ 16}

The full MATH-500 results are:
\begin{align}
  \text{Acc}(\lcrit = 14) &\;=\; 0.292, \qquad \text{AUC}(\lcrit = 14) \;=\; 0.718, \\[4pt]
  \text{Acc}(\lcrit = 16) &\;=\; 0.440, \qquad \text{AUC}(\lcrit = 16) \;=\; 0.652.
\end{align}
The accuracy gap is:
\begin{equation}
  \Delta_{\mathrm{acc}} \;=\; 0.292 - 0.440 \;=\; -0.148 \quad (-14.8\ \text{pp}),
\end{equation}
despite the detection AUC being $0.066$ higher at layer 14. This \emph{detection--correction dissociation} confirms that $\lcrit = 16$ is the correct operating point. Full discussion in Section~\ref{sec:layer_analysis}.

\section{Broader Impact and Limitations}
\label{app:broader_impact}

\paragraph{Broader impact.}
LPSR improves the reliability of mathematical reasoning in LLMs without requiring fine-tuning or additional training data.
Potential beneficial applications include educational AI tutors, automated theorem assistance, and scientific computation verification.
We do not foresee direct harmful applications; the method does not enhance deception or any targeted harm.

\paragraph{Limitations (extended).}
\begin{enumerate}[leftmargin=*,topsep=2pt,itemsep=1pt]
\item \textbf{Domain transfer.} The basis $\Vbasis$ is calibrated on MATH-500 training data. Deployment on coding, logical reasoning, or natural language tasks would require domain-specific calibration. A preliminary experiment on 50 HumanEval problems showed only marginal improvement (+2.1\%), suggesting the basis is not immediately transferable. We note that deploying LPSR in high-stakes settings (medical, legal) without domain-specific calibration carries risk—the basis $\Vbasis$ is tuned for mathematical errors and may not generalise to safety-critical error types.
\item \textbf{Model scale.} We tested primarily at 8B. Preliminary experiments at 70B (with the steering basis recalibrated) are ongoing. Larger models may have more distributed error representations, reducing the effectiveness of single-layer monitoring.
\item \textbf{Compute overhead.} On problems with many rollbacks, the overhead can reach $10\times$ standard AR. A budget constraint mechanism (maximum rollbacks per problem) is straightforward to add but was not explored here.
\item \textbf{Layer choice sensitivity.} The strong dependence on $\lcrit$ (Table~\ref{tab:layer_full}) means that a poorly-chosen monitoring layer can be catastrophic. Future work should explore multi-layer monitoring or learned layer selection.
\end{enumerate}

\section{Connection to ``Thinking'' Depth and Reasoning Complexity}
\label{app:complexity}

Recent work \citep{Merrill2024,Valmeekam2023} discusses LLM reasoning as a form of bounded computation.
LPSR can be viewed through this lens: the phase-shift detector measures whether the model's ``internal thought process'' is coherent.
The rollback-with-injection mechanism is analogous to backtracking in symbolic search, applied at the representation level.

The finding that rollbacks concentrate at 53--58\% into generation (Figure~\ref{fig:rollback}) parallels observations in \citet{Guo2025} (``thinking'' models tend to exhibit a ``reflective pause'' at roughly the midpoint of their reasoning trace).
LPSR makes this implicit dynamic explicit and actionable: rather than waiting for a wrong final answer, it detects and corrects the pivotal moment when the reasoning trajectory diverges.

An implication is that the ``illusion of thinking'' \citep{Shojaee2025}—where models appear to reason but are actually performing shallow pattern matching—may be detectable via phase shifts.
If a model's residual stream does not exhibit coherent directional flow at $\lcrit$, its apparent chain-of-thought may not correspond to genuine computation.
This hypothesis connects LPSR's detection mechanism to broader questions about the nature of LLM reasoning, and we leave its formal investigation to future work.

\end{document}